\documentclass[a4paper,12pt]{article}
\usepackage{geometry}

\usepackage{indentfirst}    
\usepackage{xcolor}
\usepackage{float} 
\usepackage{hyperref}
\hypersetup{
    bookmarksnumbered=true,  
    bookmarksopen=true,      
    citecolor=blue,         
    colorlinks,         
}
\usepackage{mathtools} 
\numberwithin{equation}{section}
\usepackage{amssymb}
\usepackage{esint} 
\usepackage{bm} 
\usepackage{mathrsfs} 
\usepackage{amsthm} 
\theoremstyle{plain}

\newtheorem{proposition}{Proposition}[section]

\newtheorem{lemma}{Lemma}[section]
\newtheorem{theorem}{Theorem}[section]
\newtheorem{definition}{Definition}[section]
\newtheorem{remark}{Remark}[section]
\theoremstyle{definition}

\usepackage{empheq}

\newcommand{\amin}[1]{\mathop{\mathrm{arg\,min}}\limits_{#1}}

\newcommand{\abs}[1]{\vert #1 \vert}
\newcommand{\dkh}[1]{\left\{ #1 \right\}} 
\newcommand{\dkhb}[1]{\big\{#1\big\}}

\newcommand{\norm}[1]{\left\Vert #1 \right\Vert} 
\newcommand{\normi}[1]{{\Vert#1\Vert}_\infty}
\newcommand{\normt}[1]{{\Vert#1\Vert}_2} 

\newcommand{\xkhb}[1]{\big(#1\big)}
\newcommand{\xkhB}[1]{\Big(#1\Big)}

\usepackage{graphicx}
\graphicspath{ {./Fig_snap/} } 
\usepackage{epsfig,epstopdf}
\usepackage{caption}
\usepackage{subcaption} 
\usepackage{array,booktabs,tabularx}
\newcolumntype{L}[1]{>{\raggedright\arraybackslash}p{#1}}
\newcolumntype{C}[1]{>{\centering\arraybackslash}p{#1}}
\newcolumntype{R}[1]{>{\raggedleft\arraybackslash}p{#1}}
\newcolumntype{Y}{>{\centering\arraybackslash}X}

\numberwithin{table}{section}
\usepackage[shortlabels]{enumitem}
\usepackage{algorithm}
\usepackage{algorithmic}

\usepackage{natbib}
\newcommand{\citee}[1]{(\citealp{#1})}
\newcommand{\algref}[1]{Algorithm~\textup{\ref{#1}}}

\newcommand{\figref}[1]{Figure~\textup{\ref{#1}}}

\newcommand{\secref}[1]{Section~\textup{\ref{#1}}}
\newcommand{\tabref}[1]{Table~\textup{\ref{#1}}}
\newcommand{\thmref}[1]{Theorem~\textup{\ref{#1}}}
\def\hat{\widehat}
\def\tilde{\widetilde}
\def\bar{\overline}
\def\0{\boldsymbol{0}}
\def\Ad{A^{\dag}}
\def\Ah{\hat{A}}
\def\Ak{A^k}
\def\Akk{A^{k+1}}
\def\Ap{A^{\dag}}

\def\bb{\bm{\beta}}
\def\bbd{\bm{\beta}^{\dag}}
\def\bbeta{\boldsymbol{\beta}}

\def\bbkk{\bb^{k+1}}
\def\bbp{{\bbeta}^{\dag}}
\def\bd{\boldsymbol{d}}
\def\bfb{\boldsymbol{b}}
\def\bP{\mathbf{P}}
\def\bX{X}
\def\d{\boldsymbol{d}}
\def\diag{\textrm{diag}}
\def\dk{\d^k}
\def\dkk{\d^{k+1}}
\def\vps{\varepsilon}
\def\etaa{\bm{\eta}}
\def\h{\boldsymbol{h}}
\def\hbb{\hat{\bm{\beta}}}
\def\hbbeta{\hat{\boldsymbol{\beta}}}
\def\hbd{\hat{\boldsymbol{d}}}
\def\hbeta{\hat{\beta}}
\def\hd{\hat{d}}
\def\la{\lambda}
\def\li{\ell_{\infty}}
\def\lt{\ell_{2}}
\def\R{\mathbb{R}}
\def\Rd{R^\dag}
\def\Rp{\mathbb{R}^{p}}
\def\sgn{\textrm{sgn}}
\def\supp{\textrm{supp}}
\def\tby{\tilde{\y}}
\def\u{\boldsymbol{u}}
\def\v{\boldsymbol{v}}
\def\w{\boldsymbol{w}}
\def\x{\boldsymbol{x}}
\def\y{\boldsymbol{y}}
\def\z{\boldsymbol z}
\begin{document}

\title{SNAP: A semismooth Newton algorithm for pathwise optimization
with optimal local convergence rate and oracle properties}

\author{
Jian Huang%
\thanks{Department of Statistics
and Actuarial Sciences,  University of Iowa.
Email: jian-huang@uiowa.edu}
\and
Yuling Jiao%
\thanks{School of Statistics and Mathematics, Zhongnan University of Economics and
Law. \hspace{3 cm}
Email: yulingjiaomath@whu.edu.cn}
\and
Xiliang  Lu%
\thanks{School of Mathematics and Statistics, Wuhan University, Wuhan, China.
Email: xllv.math@whu.edu.cn}
\and
Yueyong Shi
\thanks{School of Economics and Management, China University of Geosciences, Wuhan 430074, China. Email: yueyongshi@cug.edu.cn}
\and
and Qinglong Yang%
\thanks{School of Statistics and Mathematics, Zhongnan University of Economics and
Law. \hspace{5 cm}
Email: yangqinglong@zuel.edu.cn}
}

\date{} 

\maketitle

\begin{abstract}
We propose a semismooth Newton algorithm for pathwise
optimization (SNAP) for the LASSO and Enet in sparse,
high-dimensional linear
regression. 
SNAP is derived from a suitable formulation of
the KKT conditions based on Newton derivatives.
It solves the semismooth KKT equations efficiently
by actively and continuously seeking the support of the regression coefficients along the solution path with warm start.
At each knot in the path,  SNAP converges  {
locally superlinearly for the Enet criterion and achieves an optimal  local convergence rate for the LASSO criterion, i.e., SNAP converges in one step}  at  the  cost  of two matrix-vector multiplication per iteration. Under certain regularity conditions
on the design matrix 
and the minimum magnitude of the nonzero elements of the target regression coefficients, 
we show that SNAP hits a solution with the same signs as  the regression coefficients 
and achieves a sharp estimation  error bound in finite steps with high probability.
The computational complexity of SNAP is shown to be the same as that of LARS and coordinate descent algorithms per iteration. Simulation studies and real data analysis support our theoretical results and demonstrate that SNAP is { faster and accurate  than
LARS and coordinate descent algorithms}.

\end{abstract}


\noindent\textbf{2010 MR Subject Classification}
62F12,  
62J05,  
62J07   

\section{Introduction}

In this paper,
we propose a semismooth Newton algorithm for pathwise
optimization (SNAP) for regularized high-dimensional regression problems.
We consider the linear regression model
\begin{equation}\label{model}
 \y = X\bbeta^{\dag} + \etaa,
\end{equation}
where $\y\in\mathbb{R}^{n}$  is a response
vector,  $X \in
\mathbb{R}^{n\times p}$ is a design matrix,
$\bbeta^{\dag} =(\beta^{\dag}_{1}, \ldots ,
\beta^{\dag}_{p})^{\prime}\in \mathbb{R}^{p}$ is a vector of underlying
regression coefficients,  and $\etaa  \in \mathbb{R}^{n} $ is a vector of random errors. We assume without loss of  generality that $\y$ is centered  and the columns of $X$ are  centered and $\sqrt{n}$-normalized.
For this model, the LASSO \citee{tibshirani1996regression,chen1998atomic} solves
\begin{equation}\label{regLASSO}
 \min_{\bbeta \in \R^p} L_{\lambda}(\bbeta)
:= \frac{1}{2n}\|X\bbeta-\y\|^{2}_{2} + \lambda\|\bbeta\|_{1},
\end{equation}
where  $\lambda > 0$ is a penalty parameter.
Closely related to the LASSO is the elastic net (Enet) \citee{zou2005regularization},
which solves
\begin{equation}\label{regLASSO*}
\min_{\bbeta \in \mathbb{R}^{p}} J_{\lambda, \alpha}(\bbeta):=L_{\lambda}(\bbeta)+   
 \frac{\alpha}{2n}\|\bbeta\|^{2}_{2}, \, \alpha>0.
\end{equation}
This can be viewed as a regularized form of (\ref{regLASSO}).
Since $J_{\lambda,\alpha}(\cdot)$ is strongly  convex for $\alpha >0$, the Enet solution $\hbbeta_{\lambda, \alpha}$ is unique. This enables us to characterize
the unique minimum 2-norm LASSO solution  (\ref{regLASSO})
as the limit of $\hbbeta_{\lambda, \alpha}$ as $\alpha \to 0^+$ (Proposition  \ref{pr2}).
In high-dimensional settings, it is nontrivial to efficiently solve  (\ref{regLASSO}) and
(\ref{regLASSO*}) numerically since they are large scale nondifferentiable optimization problems.

The key ingredient of SNAP is a semismooth Newton algorithm (SNA), which is derived based
on a suitable formulation of the KKT conditions.
At each step in the iteration, the SNA works by first estimating the support of the solution based on a combination of the primal and dual information, and then finding the values of the nonzero coefficients on the support.
Interestingly, our analysis shows that
the SNA can be formally derived as a Newton algorithm based
on the notion of Newton derivatives for nondifferentiable functions
\citee{kummer1988newton,qi1993nonsmooth,ito2008lagrange}.


SNAP proceeds by running
SNA along a grid of $\lambda$ values:  $\{\lambda_{t} = \lambda_{0} \gamma^{t}\}_{t= 0, 1, ..N}$  with the continuation strategy and warm start, where $\gamma\in(0,1)$, $\lambda_{0}>0$ and the integer $N$ are user given parameters.
It is easy to implement and computationally stable. Moreover, our simulation studies indicate that SNAP is nearly problem independent, in the sense that the
computational cost 
of using SNAP to approximate the solution path is $O(Nnp)$,
independent of the following aspects of the model, including
the ambient dimension, sparsity level, correlation structure of the predictors, range of the magnitude of the nonzero regression coefficients and the noise level .

\subsection{Contributions}
The most popular algorithms for solving $\ell_1$-regularized problems
in the literature  are mainly  first order methods.
It is natural to ask whether we can develop a second order method, i.e, Newton type method, which is  a workhorse in low dimensional estimation,   for
such nonsmooth optimization problems which converges faster  than first order methods.
We give a definitive answer to this question via proposing the SNAP algorithm and  developed a MATLAB package \textsl{snap}, which is available
at \url{http://faculty.zuel.edu.cn/tjyjxxy/jyl/list.htm}.

We show that, for the LASSO, the {SNA converges locally in just one step, which is obviously the best possible local convergence rate for any algorithms (Theorem \ref{th5-6})}. For the Enet, it converges locally superlinearly (Theorem \ref{th5}).
To the best of our knowledge, these are the best convergence rates for LASSO and Enet regularized regression problems with $p \gg n$ in the literature.
Our computational complexity analysis shows that the cost of each iteration in SNA is $O(np)$, which is the same as most existing 
LASSO solvers, including LARS and coordinate descent algorithms.
Hence, the overall cost of using SNA to find the unique   minimizer of $J_{\lambda,\alpha}(\bbeta)$  
is still  $O(np)$ due to its superlinear 
convergence if it is warm started.

Another contribution of this paper is that we establish the statistical properties of SNAP in the Gaussian noise case.
Specifically, we show that under certain regularity conditions on the design matrix $X$, the solution sequence generated by SNAP enjoys the sign consistency property in finite steps if the minimum magnitude of the nonzero elements of $\bbeta^{\dag}$ is of the order $O(\sigma \sqrt{2\log(p)/n})$, which is the optimal magnitude of detectable signal. We also establish a sharp upper bound in supreme norm for the estimation error of the solution sequence.

\subsection{Related work}

\cite{osborne2000new} showed that the LASSO solution path 
is continuous and piecewise linear
as a function of $\lambda$. They proposed a Homotopy algorithm
that defines an active set of nonzero variables at the current vertex
then moves to a new vertex by adding a new variable
to or removing an existing one from  the active
set.  
\cite{efron2004least} proposed the LARS algorithm to trace the whole solution path of (\ref{regLASSO}) by omitting the removing steps in the Homotopy algorithm.
\cite{donoho2008fast} showed that,
in the noiseless case with $\etaa=0$ and under certain  conditions on $X$ and $\bbeta^{\dag}$,  LARS (Homotopy) algorithm has the ``$\|\bbeta^{\dag}\|_{0}$-step" convergence property with the cost of $O(\|\bbeta^{\dag}\|_{0}np)$.
However,  the convergence property of LARS is unknown when
the 
noise vector $\etaa$ is nonzero  in the $p > n$ settings.
Further connections   of SNA with LARS, sure independence screening \citee{fan2008sure}, and active set tricks for accelerating coordinate descent \citee{tibshirani2012strong} are  discussed in Section 5.

Several authors have adopted a Gauss-Seidel type coordinate descent algorithm (CD-GS)
\citee{fu1998penalized,friedman2007pathwise,wu2008coordinate,li2009coordinate},
as well as Jacobi type coordinate descent (CD-J),
or iterative thresholding \citee{daubechies2004iterative,she2009thresholding}
to solve (\ref{regLASSO}).
For the CD-GS proposed in  \cite{friedman2007pathwise},
the results of \cite{tseng2001convergence},
\cite{saha2013nonasymptotic} and \cite{yun2014iteration}
only ensure the  convergence and sublinear convergence rate  of the sequence of the objective functions $\{L_{\lambda}(\bbeta^{k}), k=1, 2, \ldots\}$, but not the sequence of the solutions $\{\bbeta^{k}, k=1,2, \ldots, \}$. Since in high-dimensional settings with
$p \gg n$,
the global minimizers $\hbbeta_{\lambda}$ are generally not unique, hence,  it is not clear which minimizer the  sequence $\{\bbeta^{k}, k=1,2, \ldots \}$ generated from CD-GS iterations converges to.  The CD-GS proposed in \cite{li2009coordinate} and
\cite{tseng2009coordinate} with refined sweep rules is guaranteed to converge.
Other widely used algorithms include proximal gradient descent
\citee{nesterov2005smooth,nesterov2013gradient,agarwal2012fast,xiao2013proximal},
alternative direction method of multiplier (ADMM)
\citee{boyd2011distributed,chen2017efficient,han2017linear}, among others.
For more comprehensive reviews of the literature on the related topics,
see the review papers by \cite{tropp2010computational},
and \cite{parikh2014proximal}.

\cite{agarwal2012fast} considered the statistical properties of the proximal gradient descent path. But their analysis required knowing $\|\bbeta^{\dag}\|_1$, which is unknown or hard to estimate in practice. Although this can be remedied by using the techniques developed by \cite{xiao2013proximal},  it does not achieve the sharp error bound as SNAP does.

\subsection{Notation}
Some notation used throughout this paper are defined below.
With $\|\bbeta\|_q = (\sum_{i=1}^{p}|\beta_{i}|^q)^\frac{1}{q}$  we
denote the usual $q$  $(q\in [1,\infty])$ norm of a vector $
\bbeta= (\beta_{1},\beta_{2},...,\beta_{p})^{\prime}\in \mathbb{R}^{p}$. $\|\bbeta\|_0$ denotes the number of nonzero elements of $\bbeta$.
$X^\mathrm{\prime}$ denotes the transpose of the covariate  matrix  $X \in
\mathbb{R}^{n\times p}$ and  $\norm X$ denotes the operator norm of $X$
induced by  vector with 2-norm. $\textbf{1}$  or $ \textbf{0}$ denote a
vector $\in \mathbb{R}^{p}$ or a matrix with elements all 1
or 0. Define $S =\{1,\ldots, p\}$.  For any $A, B\subseteq S$ with length
$|A|, |B|$, we
denote $\bbeta_{A}\in \mathbb{R}^{|A|} $(or $X_{A}\in \mathbb{R}^{|A|\times
p})$ as the subvector (or submatrix) whose entries (or columns) are listed
in $A$. $X_{AB}$  denotes submatrix of $X$ whose rows and columns
are listed in $A$ and $B$, respectively. We use ${\supp}(z)$, $\sgn(z)$ to denote the support and sign of a vector $\z$, respectively. We use $I$, $G$ and $\tby$  to denote the identity matrix, the regularized Gram matrix  $X^{\prime}X+\alpha I$ and $X^{\prime}\y$, respectively.

\subsection{Organization}
In Section 2 we provide a heuristic and intuitive derivation of SNA for
solving (\ref{regLASSO*}) (including (\ref{regLASSO}) as a special case by setting $\alpha=0$) and describe SNA for pathwise
optimization (SNAP). In Section 3 we establish the locally superlinear convergence rate
of SNA  for (\ref{regLASSO*}) and local one-step convergence for (\ref{regLASSO}),
and analyze the computational complexity of SNA.
In Section 4 we provide the conditions for the finite-step sign consistency of SNAP
and the upper bounds for the estimation error.
In Section 5 we discuss the relations of SNA with LARS, SIS,  and active set tricks for accelerating  coordinate descent. The implementation detail and numerical comparison with LARS and
coordinate descent methods  
are given in Section 6. We conclude in Section 7 with some comments and future work. The proofs of the main results and some background on Newton derivatives used for deriving SNA are included in the appendices.

\section{A general description of SNAP}
\label{sec:genSNAP}

In this section we first give an intuitive description of the SNA
for computing the LASSO and Enet solutions at a given $\lambda$ and $\alpha$.
We then describe the SNAP, which uses SNA for computing the
solution paths with warm start and a continuation strategy.

\subsection{Motivating SNA based on the KKT conditions}
The key idea in the proposed algorithm is to iteratively identify the active set in the optimization using  both the primal  and
dual information, then solve the problem on the active set.
Here the primal is simply $\bbeta$ and its dual is $\d =(\tby-G\bbeta)/n$.
Recall $\tby=X^{\prime}\y$ and $G=X^{\prime}X+\alpha I$.
For the LASSO, the expression of $\d$ simplifies to
$\d=X'(\y-X\bbeta)/n$, i.e., the correlation vector between the predictors and the residual.
For any given $(\lambda, \alpha)$, the KKT conditions
(Proposition \ref{th3}) assert that $\hbbeta_{\lambda,\alpha}$ is the unique Enet solution
if and only if the pair $\{\hbbeta_{\lambda,\alpha}, \hbd_{\lambda, \alpha}\}$ satisfies
\begin{eqnarray}
\left\{ \begin{array}{l}
\hbd_{\lambda,\alpha}= (\tby-G\hbbeta_{\lambda, \alpha})/n, \label{K1} \\
\hbbeta_{\lambda, \alpha} = T_{\lambda}(\hbbeta_{\lambda,\alpha} + \hbd_{\lambda,\alpha}), \label{K2}
\end{array}\right.
\end{eqnarray}
where $T_{\lambda}(\x)$ is the soft-threshold operator \citee{donoho1995adapting}
acting on $\x$ component wise, that is,
$T_{\lambda }(\x) = (T_{\lambda}(x_{1}),\ldots, T_{\lambda}(x_{p}))^{\prime}$
with
\begin{equation}\label{softth}
T_{\lambda}(x)= x - \frac{|x+\lambda|}{2}+\frac{|x-\lambda|}{2}, x \in \R.
\end{equation}

The KKT conditions in (\ref{K1})  are stated in equalities using the soft-threshold operator, rather than in the usual inequality form or in terms
of set-valued subdifferentials. This is the basis for our derivation of the SNA, which seeks to solve these nonsmooth equations.

To simplify the notation we drop the subscripts of $(\hbbeta_{\lambda, \alpha}, \hbd_{\lambda, \alpha})$ and write them as
$(\hbbeta, \hbd)$, when it does not cause any confusion.
By the second equation of (\ref{K2}) and the definition of the soft-threshold operator,  we have
\begin{equation}\label{reducer1}
\hbbeta_{B} = \textbf{0},
\end{equation}
\begin{equation}\label{reducer2}
\hbd_{A} = \lambda \sgn(\hbbeta_{A}+\hbd_{A}),
\end{equation}
where
\begin{equation}
\label{Aset1}
A = \dkh{j \in S: |\hbeta_{j} +  \hd_{j}| > \lambda} \ \text{ and } \
B = \dkh{j \in S: |\hbeta_{j} +  \hd_{j}| \leq \lambda}.
\end{equation}

Substituting (\ref{reducer1})  into the first equation of (\ref{K1})
and observing  $G_{AA}$ is invertible, we can solve
the resulting linear system to get
\begin{eqnarray}
  \hbbeta_{A} &=&  G_{AA}^{-1}(\tby_{A} - n\hbd_{A}),  \label{e29} \\
  \hbd_{B} &=& (\tby_{B}-G_{BA} \hbbeta_{A})/n.  \label{e210}
\end{eqnarray}
Therefore, $\{\hbbeta, \hbd\}$  can be obtained  from (\ref{reducer1})-(\ref{reducer2}) and (\ref{e29})-(\ref{e210}) if $A$ is known.

This naturally leads to the follow iterative algorithm for computing
$\{\hbbeta, \hbd\}$.
Let  $\{\bbeta^{k}, \d^{k} \}$ be
the primal and dual approximation of $\{\hbbeta, \hbd\}$ at the $k${th} iteration.
Based on (\ref{Aset1}), we approximate the active and inactive sets by
\begin{equation}
A_{k} = \dkh{j \in S: |\beta^{k}_{j} +  d^{k}_{j}| >\lambda}
\ \mbox{ and } \
B_{k} = \dkh{j \in S: |\beta^{k}_{j} +  d^{k}_{j}| \leq \lambda}.\label{eac}
\end{equation}

Based on (\ref{reducer1})-(\ref{reducer2}) and (\ref{e29})-(\ref{e210}) we obtain the updated approximation $\{\bbeta^{k+1}, \d^{k+1}\}$,
\begin{align}
\bbeta^{k+1}_{B_{k}} &= \textbf{0},  \label{e211}\\
 \d^{k+1}_{A_{k}} &= (\lambda-\bar{\lambda}) \sgn(\bbeta^{k}_{A_{k}}+\d^{k}_{A_{k}}),   \label{e212} \\
 \bbeta^{k+1}_{A_{k}} &=  G_{A_{k}A_{k}}^{-1}(\tby_{A_{k}} - n\d_{A_{k}}^{k+1}),  \label{e213}\\
 \d_{B_{k}}^{k+1} &= (\tby_{B_{k}}-G_{B_{k}A_{k}} \bbeta^{k+1}_{A_{k}})/n.   \label{e214}
\end{align}
In \eqref{e212} we introduce a (small) shifting
parameter $\bar\lambda$ with $0\le\bar\lambda<\lambda$ and use a slightly more
general version of (\ref{reducer2}), replacing $\lambda$ with $\lambda-\bar\lambda$ in (\ref{reducer2}).
For $\bar\lambda > 0$, we solve a less shrunk version of
the Enet. For the solution sequence $\{\bbeta^{k}, k\ge 1\}$ with a suitable $\bar\lambda>0$, we show that it achieves finite-step sign consistency and sharp estimation error bound (Theorem \ref{th6}).

Summing up the above discussion, we get the SNA for minimizing (\ref{regLASSO*})
in Algorithm \ref{alg1} below, where we write $\hbbeta(\lambda)=\hbbeta_{\lambda,\alpha}$ for a fixed $\alpha$.

\begin{algorithm}[H]
\caption{
$(\hbbeta(\lambda), \hbd(\lambda))\longleftarrow \text{SNA}(\bbeta^{0}, \d^{0}, \lambda, \bar{\lambda}, K)$}\label{alg1}
\begin{algorithmic}[1]
\STATE Input:  $ X, \y,\alpha, \lambda, \bar{\lambda},K$,  initial guess  $ \bbeta^{0}, \d^{0}$, $A_{-1} = \textrm{supp}( \bbeta^{0})$. Set $k=0$.
\STATE Compute $\tilde{\y}=X^{\prime}\y$ and store it.
\FOR{$k= 0,1,\cdots, K$}
\STATE Compute $ A_{k}, B_{k}$ using \eqref{eac}.
\STATE If $A_{k} = A_{k-1}$ or $k\geq K$.

       \quad \quad Stop  and denote the last iteration by $\bbeta_{\hat{A}}, \bbeta_{\hat{B}}, \d_{\hat{A}}, \d_{\hat{B}}.$

       Else

\STATE  Compute $\{\bbeta^{k+1}, \d^{k+1}\}$ using \eqref{e211} - \eqref{e214}.

      \quad \quad  $k:=k+1.$

       End
\ENDFOR
\STATE Output: ${\hbbeta(\lambda)}= \left(
                               \begin{array}{c}
                                 \bbeta_{\hat{A}} \\
                                 \bbeta_{\hat{B}} \\
                               \end{array}
                             \right)$ \mbox{ and }
                              ${\hbd(\lambda)}= \left(
                               \begin{array}{c}
                                 \d_{\hat{A}} \\
                                 \d_{\hat{B}} \\
                               \end{array}
                             \right)
$
\end{algorithmic}
\end{algorithm}

\begin{remark}
In the algorithm, we use a safeguard maximum number of iterations $K$ that can be defined by the user. We usually set $K \le 5$ due to the locally superlinear/one-step convergence of SNA.
\end{remark}

Each line in Algorithm  \ref{alg1} consists of  simple vector and matrix  multiplications, except \eqref{e213} in line 6,
where we need to invert a $|A_{k}|\times|A_{k}|$
matrix.
Note that $A_{k}$ is usually a small subset of $S$ if Algorithm  \ref{alg1} is warm started. Intuitively, at the $k${th} step in the iteration, this algorithm tries to identify $A_k$, an approximation of the underlying support
by using the estimated coefficients with a proper
adjustment $\d^k$ determined by the KKT, and solves a low-dimensional
adjusted least squares problem on $A_k$. Therefore, with a good starting estimation
of $A_k$, which is guaranteed by using a continuation strategy with warm start
described below, Algorithm \ref{alg1} can find a good solution in a few steps.
In Section 3, we derive Algorithm \ref{alg1} formally from the semismooth Newton method and show that its convergence rate is locally superlinear for the Enet and locally one step for the LASSO.

\subsection{Solution path approximation}
\label{ssec:solpath}

We are often interested in the whole solution path $\hbbeta(\lambda)\equiv \hbbeta_{\lambda, \alpha}$ of (\ref{regLASSO*})
for $\lambda \in [\lambda_{\min}, \lambda_{\max}]$ and some given $\alpha\ge 0$.
Here we approximate the solution path by computing $\hbbeta(\lambda)$ on a given  finite set
$\Lambda =\{\lambda_{0},\lambda_{1}...,\lambda_{N}\}$ for some integer $N$, where $\lambda_0 > \cdots > \lambda_{N} > 0$.
Obviously, $\hbbeta(\lambda)=\textbf{0}$ satisfies (\ref{K1}) and (\ref{K2}) if
$\lambda\geq\|X^{\prime}\y/n\|_{\infty}$.
Hence we set $\lambda_{\max}=\lambda_{0}=\|X^{\prime} \y/n\|_{\infty}$,
$\lambda_{t} = \lambda_{0} \gamma^{t}, t =0, 1, ..., N$, and
 $\lambda_{\min}=\lambda_{0}\gamma^{N}$, where $\gamma\in (0,1)$.

We adopt a simple continuation technique with warm start
in computing the solution path. This strategy has been successfully
used  for computing the LASSO and Enet
paths \citee{friedman2007pathwise, jiao2017iterative}.
We use the solution at $\lambda_{t}$ as the initial value for computing the solution at $\lambda_{t+1}$.
{The shift parameter $\bar{\lambda}$  can be   vary  at different path knots $\lambda_t$,  so here we use $\bar{\lambda}_t$ to demonstrate  this}.
We summarize this in the following  SNAP algorithm

\begin{algorithm}[H]
\caption{
 $\hbbeta(\Lambda)\longleftarrow \text{SNAP}(\lambda_{0},\gamma,N,K)$}\label{alg2}
\begin{algorithmic}[1]
\STATE Input:  $\lambda_{0}= \|X^{\prime} y/n\|_{\infty},\hbbeta(\lambda_{-1}) = \textbf{0}, \hat{\d}(\lambda_{-1})=X^{\prime} \y/n,\gamma,N,K.$
\FOR{$t= 0,1...N.$}
\STATE Set $\lambda_t = \lambda_0\gamma^t$ and $(\bbeta^0,\d^0) = (\hbbeta(\lambda_{t-1}), \hat{\d}(\lambda_{t-1}))$.
\STATE $(\hbbeta(\lambda_{t}), \hat{\d}(\lambda_t)) \longleftarrow \text{SNA}(\bbeta^{0}, \d^{0}, \lambda_t, \bar{\lambda}_{t},K)$
\ENDFOR \STATE
Output: $\hbbeta(\Lambda) = [\hbbeta(\lambda_0),...,\hbbeta(\lambda_N)].$
\end{algorithmic}
\end{algorithm}

When running SNAP with warm start (Algorithm \ref{alg2}),
SNA (Algorithm \ref{alg1}) usually converges {in a few steps,
since SNA converge locally superlinearly or locally in one step  and warm start provides a good initial value}.
\section{Derivation of SNA and convergence analysis}


\subsection{KKT conditions}
\label{KKT-subsection}
In this subsection, we first discuss the relationship between the
minimizers of (\ref{regLASSO}) and (\ref{regLASSO*}). We then
characterize the unique minimizer (\ref{regLASSO*}) by its KKT system.

\begin{proposition}\label{th1}
Let $M_{\lambda}$ be the set of the LASSO solutions given in
(\ref{regLASSO}).  Then $M_{\lambda}$ is nonempty, convex and compact.
\end{proposition}

In general, the uniqueness of the LASSO solution (the minimizer of  (\ref{regLASSO}))
cannot be guaranteed in the $p \gg n$ settings.
But the one in $M_{\lambda}$ with the minimum Euclidean norm denoted by $\hbbeta_{\lambda}$ is unique. We have the following relation between $\hbbeta_{\lambda}$ and the Enet solutions $\hbbeta_{\lambda, \alpha}$.

\begin{proposition}\label{pr2}
For $\alpha > 0$, the Enet (\ref{regLASSO*}) admits a unique minimizer denoted by
$\hbbeta_{\lambda, \alpha}$. Furthermore,
$\|\hbbeta_{\lambda, \alpha}-\hbbeta_{\lambda}\|_{2}\rightarrow 0$ as $\alpha
\rightarrow 0^{+}$.
\end{proposition}

By Proposition \ref{pr2}, a good numerical solution of (\ref{regLASSO*}) is a good approximation of the minimum 2-norm  minimizer of (\ref{regLASSO})
for a sufficiently small $\alpha$.

\begin{proposition}\label{th3}
Let $\hbbeta_{\lambda, \alpha} \in \Rp$ be the Enet solution, which is
the unique minimizer of $J_{\lambda,\alpha}$ in (\ref{regLASSO*})
for $\alpha >0$. Then there exists a
$\hbd_{\lambda, \alpha} \in \Rp$ such that (\ref{K1}) hold.
Conversely, if there exists  $\hbbeta_{\lambda, \alpha}\in \Rp$ and
$\hbd_{\lambda, \alpha}
\in \Rp$ satisfying (\ref{K1}), then $\hbbeta_{\lambda, \alpha}$ is the
unique minimizer of $J_{\lambda,\alpha}$ in (\ref{regLASSO*}).

The KKT equations (\ref{K1}) with $\alpha=0$ also characterize
the LASSO solution (\ref{regLASSO}), except that  the solution may not be unique.
\end{proposition}

Here the KKT conditions are formulated
in terms of equalities (\ref{K1}), which are different from  but equivalent to the usual inequality form
$$
\hbd_{A} = \lambda \sgn(\hbbeta_{A}),$$
$$
\|\hbd_{A^c}\|_{\infty}\leq \lambda,$$
where, $\hbd = (X^{\prime}\y-G\hbbeta)/n$ and $A = \textrm{supp}(\hbbeta).$
The reason we adopt the equation form (\ref{K1}) is that
we can transform  the minimization problem (\ref{regLASSO*}) into a root finding problem, which help us derive SNA formally under the framework of semismooth Newton method.

\subsection{SNA as a Newton algorithm}
\label{SNA-subsection}
We now formally derive the SNA based on the KKT conditions by using the semismooth Newton method  \citee{kummer1988newton,qi1993nonsmooth,ito2008lagrange}
for finding a root of a nonsmooth equation. This enables us to prove its locally superlinear convergence stated in
Theorem \ref{th5} below. The definition and related property  on Newton derivative  are given in Appendix A.

Let
$$\z=\left(
         \begin{array}{c}
         \bbeta \\
          \d  \\
         \end{array}
       \right) \ \mbox{ and } \
F(\z)=
\left[\begin{array}{c}
 F_{1}(\z)\\
F_{2}(\z)
\end{array}\right]
: \Rp \times \Rp \to  \mathbb{R}^{2p},
$$
where $$F_{1}(\z):=\bbeta  - T_{\la}(\bbeta + \d),$$
$$F_{2}(\z):= G \bbeta + n\d -\tby.$$
By Proposition \ref{th3}, to find the  minimizer of (\ref{regLASSO*}), it
suffices to find a root of $F(\z)$.
Although the classical Newton algorithm cannot be applied directly since $F(\z)$ is not Fr\'{e}chet differentiable, we can resort to semismooth Newton algorithm since $F(\z)$ is Newton differentiable.

Let $$A := \dkh{i \in S: |\beta_{i} + d_{i}| \geq  \lambda}, \quad B := \dkh{i \in S: |\beta_{i} + d_{i}| < \lambda}.$$
We reorder  $(\bbeta',\d')'$ such that
$\z=(\d_{A}', \bbeta_{B}', \bbeta_{A}', \d_{B}')'$.
We also reorder $F_{1}(\z)$ and $F_{2}(\z)$ accordingly,
$$F(\z)=
\left[\begin{array}{c}
 \bbeta_{A}  - T_{\la}(\bbeta_{A} + \d_{A}) \\
  \bbeta_{B}  - T_{\la}(\bbeta_{B} + \d_{B}) \\
  G_{AA} \bbeta_{A} + G_{AB} \bbeta_{B}+ n\d_{A} -{\tby}_{A}      \\
G_{BA} \bbeta_{A} + G_{BB} \bbeta_{B}+ n\d_{B} -{\tby}_{B}
\end{array}\right].$$
We have the following result concerning the Newton derivative of $F$.

\begin{theorem}\label{th4}
$F(\z)$ is Newton differentiable at any point $\z$. And
\begin{equation}
H:=
\begin{bmatrix}
& -I_{AA}    & \textbf{0}                                                &\textbf{0}                                             & \textbf{0} \\ \\
& \textbf{0}          & I_{BB}                                                  &\textbf{0}                                     &\textbf{0} \\ \\
& nI_{AA}    & X_{A}^{\prime}X_{B}                                   &G_{AA}                          & \textbf{0} \\ \\
& \textbf{0}          & G_{BB}            &X_{B}^{\prime}X_{A}                         & nI_{BB}
\end{bmatrix}
\in \nabla_{N}F(\z).
\end{equation}
Furthermore, $H$ is invertible and $H^{-1}$ is uniformly bounded
with
\[
\|H^{-1}\|\leq 1+2(n+1+\alpha+\norm X^2)^2/\alpha.
\]
\end{theorem}

At the $k_{th}$ iteration, the semismooth Newton method
for finding the root of $F(\z)=\0$ consists of two steps.

\begin{enumerate}
\item[(1)] Solve $H_{k}D^{k}=-F(\z^{k}) $ for $D^{k}$, where
$H_{k}$ is an element of $\nabla_{N}F(\z^{k})$.

\item[(2)] Update $\z^{k+1} = \z^{k} + D^{k}$, set $k \leftarrow k+1$
and go to step (1).
\end{enumerate}

This has the same form as the classical Newton method, except that
here we use an element of $\nabla_{N}F(Z^{k})$ in step (1).
Indeed, the key to the success of this method is to find a suitable
and invertible $H_k$.
We state this method in Algorithm \ref{alg3}.

\medskip
\begin{algorithm}[H]
\caption{SNA for finding a root of  $F(\z)$ }\label{alg3}
\begin{algorithmic}[1]
\STATE Input:  $ X, \y,\lambda,  \alpha$,  initial guess  $ \z^{0}=\left(
         \begin{array}{c}
         \bbeta^{0} \\
         \d^{0}  \\
         \end{array}
       \right)
.$ Set $k=0$.
\FOR{$k=
0,1,2,3,\cdots$} \STATE Choose $H_{k}\in \nabla_{N}F(\z^{k})$. \STATE
Get the semismooth Newton direction
$D_{k}$ by solving
\begin{equation}\label{ssnd}
H_{k}D^{k}=-F(\z^{k}).
\end{equation}
 \STATE Update
 \begin{equation}\label{ssnup}
\z^{k+1} = \z^{k} + D^{k}.
\end{equation}
\STATE Check Stop condition

       If stop

       Denote the last iteration by $\hat{\z}.$

       Else

       $k:=k+1.$
\ENDFOR \STATE Output:
$\hat{\z}$
as an estimate of the roots of $F(\z)$.
\end{algorithmic}
\end{algorithm}

\begin{remark}
When $A_{k} = A_{k+1}$ holds for some $k$, Algorithm \ref{alg1} converges. Hence  it is natural to   stop Algorithm \ref{alg1} accordingly.  A common   condition that can be used as a stop rule of Algorithm \ref{alg3} is when $\|F(\z^{k})\|_{2}$ is sufficiently small,
since this algorithm is a root finding process. Therefor we can use both stopping rules in Algorithm \ref{alg1} due to the equivalence of Algorithm \ref{alg1} and Algorithm \ref{alg3}. We also stop Algorithm \ref{alg1} when the iteration number $k$ exceeds a prespecified integer $K$.
\end{remark}

It can be verified that Algorithm \ref{alg1}  with $\bar{\la} = 0$ is just Algorithm \ref{alg3} written in a form for easy computational implementation. The details are given in \text{Appendix C}. Thus it is indeed a semismooth Newton method.
The more compact form of Algorithm \ref{alg3} is better suited for its convergence analysis.

\begin{theorem}\label{th5}
Let $H_{k}$ in Algorithm \ref{alg3} be given in  (\ref{ndk}).
Then the sequence $\{\bbeta^{k}, k=1, 2, \ldots\}$ generated based on Algorithm \ref{alg3} (and Algorithm \ref{alg1} with $\bar{\la} = 0$)  converges locally and superlinearly to $\hbbeta_{\la, \alpha}$, the unique minimizer of (\ref{regLASSO*}).
\end{theorem}
{
Theorem \ref{th5} shows the local supperlinear convergence rate of SNA, which is a superior  property of Newton type algorithms to first order methods.

\begin{theorem}\label{th5-6}
For a give $\la >0$, let $\hbbeta\equiv \hbbeta_{\la}$ be  a minimizer of   (\ref{regLASSO}), $\hbd = X^{\prime}(\y-X\hbbeta)/n$, $A =
\{i: |\hbeta_i + \hd_i|> \lambda \}$, $\tilde{A} =
\{i: |\hbeta_i + \hd_i| \neq \lambda \}$, and
  $C = \min_{i\in \tilde{A}} ||\hbeta_i + \hd_i| - \lambda| > 0.$
Suppose $\textrm{rank}(X_{A})=|A|$ and
 the initial guess   $\bbeta^0,\bd^0$ satisfies
$\|\hbbeta - \bbeta^0\|_{\infty} + \|\hbd - \bd^0\|_{\infty} \leq C.$
Then,
 $\bbeta^1 =\hbbeta$, where $\bbeta^1$ is generated by  Algorithm \ref{alg3} with  $\alpha = 0$ and $\bar{\la} = 0$.
\end{theorem}
}

Theorem \ref{th5-6} shows that the SNA has an optimal local convergence rate in the sense that it converges in just one step, which improves the locally supperliner convergence rate of semismooth Newton method, see for example,
\cite{kummer1988newton}, \cite{qi1993nonsmooth} and \cite{ito2008lagrange}.

\color{black}

\subsection{Computational complexity analysis}

We now consider the computational complexity of SNA (Algorithm \ref{alg1}).
We look at the number of floating point operations per iteration.
Clearly it takes  $O(p)$ flops to finish step 4-7 in  Algorithm \ref{alg1}.
For step 8, we solve the linear equation iteratively by  conjugate gradient (CG) method  initialized with the projection of the previous solution onto the current active set
\citee{golub1996matrix}. The main operation of  iteration of CG is two matrix-vector multiplication cost $2n|A_{k}|$ flops
Therefore we can control the number of CG iterations smaller than $p/(2|A_{k+1}|)$ to make
that  $O(np)$ flops will be enough for   step 8.
For step 9, calculation of the matrix-vector product  costs  $np$ flops.
So the the overall cost per iteration of Algorithm 1 is $O(np)$ which is also the cost for state-of-the-art first order  LASSO solvers.
The local  superlinear/one step convergence of SNA guaranteed that a good solution can be found in only a few iteration    if it is warm started.
Therefore, at each knot of the path,  the whole cost of SNA can be still $O(np)$ if we use the continuation strategy.
So Algorithm \ref{alg2} (SNAP) can get the solution path accurately and efficiently at the cost of $O(Nnp)$ with $N$ be the number of knot on the path,  see the numerical results in Section 6.

\section{Error bounds and finite-step sign consistency}
As shown in Theorems  \ref{th5} and \ref{th5-6}, SNA converges locally superlinearly for Enet and converges in one step for LASSO.  In this section we prove that the simple warm start technique makes the SNAP converge globally under certain mutual coherence conditions on $X$ and a condition on the minimum magnitude of the nonzero components of $\bbeta^{\dag}$.
Specifically, we show that SNAP hits a solution with the same sign as  $\bbeta^{\dag}$
and attains a sharp statistical error bound  in finitely many steps with high probability if we properly design the path $\{\lambda_{t} = \lambda_{0} \gamma^{t}\}_{t= 0, 1, ..N}$ and run SNA along it with warm start.

We only consider the LASSO, so we set $\alpha = 0$ and $G = X^{\prime}X$.
The mutual coherence $\nu$ defined as $\nu = \max_{i\neq j}\abs{G_{i,j}}/n$
\citee{donoho2001uncertainty,donoho2006stable}
characterizes the minimum angle between different columns of $X/\sqrt{n}$.
Let $ A^{\dag} = \textrm{supp}(\bbeta^{\dag})$ and $ T = |A^{\dag}|$.
Define
$\abs{\bbeta^{\dag}}_{{\min}} = \min \{|\beta_j^{\dag}|: j \in A^{\dag}\},$
Denote the universal threshold value by $\lambda_u = \sigma \sqrt{2\log(p)/n}$.
Let   $\delta_u = 3\lambda_{u}$,
$\lambda_{0}=\|X^{\prime} y/n\|_{\infty}$,
 and
$\lambda_{t} = \lambda_{0} \gamma^{t}$, $t= 0, 1, ...$.

We make  the following  assumptions on the design matrix $X$, the target coefficient $\bbeta^{\dag}$, and the noise vector $\etaa$.
\begin{itemize}
\item[(A1)] The mutual coherence satisfies $T\nu  \leq \frac{1}{4}$.
\end{itemize}
\begin{itemize}
\item[(A2)] The smallest nonzero regression coefficient satisfies $|\bbeta^{\dag}|_{min}\geq 78\la_u $.
\end{itemize}
\begin{itemize}
\item[(A3)] $\etaa$ satisfies  $\etaa\sim N(0, \sigma^2 I_n)$.
\end{itemize}

\begin{lemma}\label{Lem3}
Suppose that (A1) to (A3) hold.
There exists an integer $N \in [1, \log_{\gamma}(\frac{10\delta_u}{\lambda_{0}}))$ such that $\lambda_{N}  > 10\delta_u \geq \lambda_{N+1}$ and $|\bbeta^{\dag}|_{min}>8\la_N/5$ hold with  probability at least  $1 - {1}/{(2\sqrt{\pi\log(p)})}$.
\end{lemma}

\begin{theorem}\label{th6}
Suppose that (A1) to (A3) hold.
Then with probability at least $1 - {1}/{(2\sqrt{\pi\log(p)})}$, $\text{SNAP}(\lambda_{0},\gamma,N,K)$ with $\gamma = 8/13$, $N$  determined in Lemma  \ref{Lem3}, $K\geq T$, and $\bar{\la} = \frac{9}{10}\la_t+\delta_u$ at the $t_{th}$ knot, has a finite step sign consistence property and achieves a sharp estimation error, i.e.,
\begin{equation}\label{Th61}
 \sgn(\hbbeta(\lambda_N)) =\sgn(\bbeta^{\dag}),
 \end{equation}
  and
  \begin{equation}\label{Th62}
 \|\hbbeta(\lambda_N)-\bbeta^{\dag}\|_{\infty }<\frac{23}{6}\la_u.
  \end{equation}
\end{theorem}

\begin{remark}
 From the proof of Theorem \ref{th6} we can see that SNAP (Algorithm  \ref{alg3}) with  $\bar{\lambda}=0$ can recover $\bbeta^{\dag}$ exactly by letting $\lambda_t\rightarrow0$ in the case  $\etaa = \0$. However, if the observation contains noise
 we have to set the shift parameter $\bar{\lambda}$ in SNAP to be nonzero which reduce the
 amount of shrinkage of LASSO.
 \end{remark}

The properties of 
LASSO have been studied by many authors.
For example, \cite{zhao2006on} and \cite{meinshausen2006high}
showed that LASSO is sign consistent under a strong irrepresentable condition, which is a little weaker than (A1). They also required $|\bbp|_{min}$ be bounded below by  $O(n^{-c/2})$ with $c\in(0,1)$,  which  is stronger than (A2).
\cite{zhang2008the} required $X$ satisfy a sparse Rieze condition, which may be weaker than (A1);  and $|\bbp|_{min}$ larger than  $O(\sqrt{T}\la_u)$, which is stronger than (A2).
\cite{wainwright2009sharp} assumed a condition stronger than the strong irrepresentable condition to guarantee the uniqueness of LASSO and its sign consistency with a condition on $|\bbp|_{min}$ similar to (A1).
\cite{lounici2008sup} and \cite{candes2009near} assumed the mutual coherence conditions
with $T\nu<{1}/{7}$ and $\nu<{c}/{\log(p)}$ for a constant $c$,  respectively; and
their requirements for $|\bbeta^{\dag}|_{min}$  are similar to (A1) with different constants.
In deriving the $\ell_2$ and $\ell_{\infty}$ error bounds of the LASSO,
\cite{donoho2006stable} and \cite{zhang2009some} assumed
$T\nu<{1}/{4}$ and $T\nu\leq {1}/{4}$, respectively. The latter is exactly (A1).
However, these existing results do not imply the finite-step sign consistency property
established in Theorem \ref{th6}.

All the results mentioned above concern the minimizer of the  LASSO problem,  but they  did not directly address the statistical properties of the sequence generated by a specific solver. So there is a gap between those theoretical results and the computational solutions. 
There has been efforts to close this gap. For example, \cite{agarwal2012fast} considered the statistical properties of the proximal gradient descent path. But their analysis required knowing $\|\bbeta^{\dag}\|_1$, which is hard to estimate in practice.
 \cite{xiao2013proximal} remedied this, but their result does not achieve the sharp error bound like SNAP does. Also, the technique used for deriving the statistical properties of SNAP (a Newton type  method), is quite different from the proximal gradient method (a gradient type method).


\section{Connections  LARS, SIS and active set tricks for accelerating coordinate descent}
The key idea in SNAP is using the Newton type method  SNA  to iteratively identify the active set  using  both the primal  and
dual information, then solve the problem on the active set.
In this section we discuss the connections between SNAP with other three dual active set mehtods, i.e.,
 LARS, SIS \citee{fan2008sure}, and sequential strong rule SSR (active set tricks for accelerating coordinate descent) \citee{tibshirani2012strong}.

LARS \citee{efron2004least} also does not solve (\ref{regLASSO}) exactly
since it omits the removing procedure of Homotopy \citee{osborne2000new}.
As discussed in \cite{donoho2008fast}, the LARS  algorithm
can be formulated  as
\begin{align*}
\bbeta_{B_{k}}^{k+1} &= \textbf{0}, \\
  \bbeta_{A_{k}}^{k+1} &=(X_{A_{k}}^{\prime} X_{A_{k}})^{-1}
(\tby_{A_{k}} - \bar{\lambda}^{k}\sgn (\d^{k}_{A_{k}})),
\end{align*}
where $A_{k}$ is the set of the indices of the variables with highest correlation with the current residual, $B_{k} = (A_{k})^{c}$, $\bar{\lambda}^{k} = \|\d^{k}\|_{\infty} -\gamma_{k}$, $\d^k = X^{\prime}(\y-X\bbeta^{k})/n$,
and $\gamma_{k}$ is the step size to the next breakpoint on the
 path \citee{efron2004least,donoho2008fast}.
 Comparing  Algorithm \ref{alg2} (SNAP) (by setting $K= 0$)  with the above reformulation of the LARS algorithm, we see that both SNAP and LARS can be understood as approaches for estimating the support of the underlying solution, which is the essential aspect in fitting sparse, high-dimensional models.
So, SNAP and LARS share some similarity both in formulation and in spirit although they were derived from different perspectives.
However, the definitions of the active set in SNAP  is based on the sum of  primal approximation (current approximation $\bbeta^k$) and the dual approximation (current correlation  $\dk = X^{\prime}(\y-X\bbeta^k)/n$) while  LARS is based on  dual only. The   following low-dimensional small  noise interpretation may clarify the difference between the two active set definitions.  If  $X^{\prime}X/n \approx $ identity  and  $\etaa \approx\textrm{0}$
we get \begin{equation*}
\d^k = \bX^{\prime}(\y-\bX\bbeta^k)/n = \bX^{\prime}(\bX\bbeta^{\dag} + \etaa -\bX\bbeta^k)/n \approx \bbeta^{\dag}-\bbeta^k +\bX^{\prime}\etaa/n \approx \bbeta^{\dag}-\bbeta^k
\end{equation*}
and
$$\bbeta^k+\d^k\approx \bbeta^{\dag}.$$
In addition, LARS selects variables one by one while SNAP can selects more than one variable at each  iteration. Also, the adjusted least squares fits on the active sets
in SNAP and LARS are different.
Under certain conditions on $X$  both of them recover $\bbeta^{\dag}$ exactly  in the noise free case \citee{donoho2008fast} even when $p>n$. But the convergence or consistency  of LARS is unknown when the noise vector $\etaa\neq \0$.

Given a starting point $\lambda_0$, SNA is initialized with
$\bbeta^0=\textbf{0}, \d^0= X'\y/n$.
Therefore,
\[
A_{0}=\{j: |\beta_j^0+d_j^0| > \lambda_0\} =\{j: |\x_j'\y/n| > \lambda_0\}.
\]
Thus the first active set  generated by SNAP   contains the features that coorelated with y larger than $\lambda_0$,   which are the same
as those from the sure independence screening \citee{fan2008sure}
 with parameter  $\lambda_0$ and include the one  selected by  the first step of   LARS. We then use \eqref{e211} - \eqref{e214} to obtain
$\{\bbeta^1, \d^1\}$, and  update the
active set to $A_1$  using \eqref{eac}. Clearly, for $k \ge 1$, $A_k$ are determined not just by the correlation $\d^k$,  but  by  the primal ($\bbeta^k$) and dual ($\d^k$) together.

\cite{tibshirani2012strong} proposed  a sequential strong rule (SSR) for discarting predictors in LASSO-type problems.
At point $\lambda_t$ on the solution path, this rule discards the $j$th predictor if
\[
|\hd_{j}(\lambda_{t-1})| < 2\lambda_t-\lambda_{t-1},
\]
where $\hd_j(\lambda) = \x_j^{\prime}(\y-X\hbbeta(\lambda))/n$ for the LASSO penalty.
They define  active set
\[
A_{k} = \{j: |d_{j}(\lambda_{t-1})| \geq 2\lambda_t-\lambda_{t-1}\},
\]
and  set $\hbbeta(\lambda_t)_{B_{k}}=0$ for $B_{k}=A_{k}^c$
and solve the LASSO problem
on $A_{k}$. By combining with a simple check of  the KKT condition, it speeds up the computation considerably.
So SNA shares some similarity in spirit with SSR 
in that both methods seek to identify an active set and solve a smaller optimization
problem, although they are derived from quite different perspectives. However, there are some important differences. First, the active sets are determined differently. Specifically, SSR determines the active set
only based on the dual approximation; while SNA uses both primal and dual approximation.  Second, SNA does not need the unit slope assumption, and additional check of the KKT conditions is not needed (The cost of check KKT is $O(np)$). Third, as far as we know, the statistical properties of the solution sequence generated from SSR are unknown, while
error bounds and sign consistency are established under suitable conditions
for the solution sequence generated from the SNAP.


\section{Numerical examples}
\label{sec:numEx}

In this section, we present numerical examples to
evaluate the performance of the proposed SNAP algorithm \ref{alg2}
for solving LASSO.
All experiments are performed
in MATLAB R2010b on a quad-core laptop with an Intel Core i5 CPU (2.60 GHz)
and 8 GB RAM running Windows 8.1 (64 bit).

\subsection{Comparison with existing popular algorithms}
Both the LARS \citee{efron2004least,donoho2008fast}
and the CD \citee{friedman2007pathwise,friedman2010regularization}
are popular algorithms capable of efficiently computing
the LASSO solution, hence we compare the
proposed SNAP with these two algorithms.
In implementation, we consider two solvers:
(1) SolveLasso, the Matlab code for LARS with the LASSO modification,
available online at
\url{http://sparselab.stanford.edu/SparseLab_files/Download_files/SparseLab21-Core.zip};
(2) glmnet, the Fortran based Matlab package using CD,
available online at
\url{https://github.com/distrep/DMLT/tree/master/external/glmnet}.
The parameters in the solvers are the default values as their online versions.
In addition to the default stopping parameters in the solvers,
we stop LARS (SolveLasso), CD (glmnet) and SNAP if the number of nonzero elements at some
iteration is larger than a given fixed quantity
such as $n/\log(p)$ or even larger $0.5n$,
since the  upper bound of the estimated sparsity level of  LASSO is $O(n/\log(p))$
when $n\ll p$ \citee{candes2006robust,candes2006near}.

\subsection{Tuning parameter selection}
To choose a proper value of $\lambda$ in \eqref{regLASSO}
is a crucial issue for LASSO problems,
since it  balances the tradeoff between the data fidelity
and the sparsity level of the solution.
In practice, the Bayesian information criterion (BIC)
is a widely used selector for the tuning parameter selection,
due to its model selection
consistency under some regularity conditions.
We refer the readers to
\cite{wang2007tuning,chen2008extended,wang2009shrinkage,
chen2012extended,kim2012consistent,wang2013calibrating}
and references therein for more details.
In this paper, we  use a modified BIC (MBIC) from \cite{kim2012consistent}
to choose $\lambda$, which is given as
\begin{equation}\label{mbic}
\hat{\lambda}=\amin{\lambda\in\Lambda}
\dkh{\frac{1}{2n}\|X\hat{\bb}(\la)-\y\|^{2}_{2}
+|\hat{A}(\la)|\frac{\log(n)\log(p)}{n}},
\end{equation}
where $\Lambda=\{\lambda_t\}_t$ is the candidate set for $\lambda$,
and $\hat{A}(\la)=\{j:\hat{\beta}(\lambda)\neq 0\}$
is the model identified by $\hat{\bb}(\la)$. Besides,
the high-dimensional BIC (HBIC) in \cite{wang2013calibrating} defined by
\begin{equation}\label{hbic}
\hat{\lambda}=\amin{\lambda\in\Lambda}
\dkh{\log\xkhb{\|X\hat{\bb}(\la)-\y\|^{2}_{2}/n}
+|\hat{A}(\la)|\frac{\log(\log n)\log(p)}{n}}
\end{equation}
is also a good candidate for the selection of $\lambda$.
Unless otherwise specified, the MBIC \eqref{mbic} is the default one
to select $\lambda$.

\subsection{Simulation}

\subsubsection{Implementation setting}
\label{ssec:sett}
The $n\times p$ design matrix $X$ is generated as follows.
\begin{itemize}
\item[(i)]
Classical Gaussian matrix with correlation parameter $\rho$.
The rows of $X$ are drawn independently from $N(0,\Sigma)$
with $\Sigma_{jk} = \rho^{|j-k|}, 1 \le j, k \le p$, $\rho\in(0,1)$.
\item[(ii)] Random Gaussian matrix  with auto-correlation parameter $\nu$. First we
generate a random Gaussian matrix $\widetilde{X} \in\mathbb{R}^{n\times p}$ with its
entries following i.i.d. $N(0,1)$. Then we define a matrix
$X\in\mathbb{R}^{n\times p}$ by setting $X_1 = \widetilde{X}_1$,
\begin{equation*}
   X_j = \widetilde{X}_j + \nu*(\widetilde{X}_{j-1} + \widetilde{X}_{j+1}), \ \ j=2,...,p-1,
\end{equation*}
and $X_p= \widetilde{X}_p$.
\end{itemize}
 The elements of the
error vector $\etaa$ are generated independently with
$\eta_{i}\thicksim N(0,\sigma^2)$, $i=1,2,...,n$.
Let $A^{\dag}=\textrm{supp}(\bbeta^{\dag})$ be the support of $\bbeta^{\dag}$,
and let $ R^{\dag}=
{\max\{|\bbeta^{\dag}_{A^{\dag}}|\}}/{\min\{|\bbeta^{\dag}_{A^{\dag}}|\}}$
be the range of magnitude of nonzero elements of $\bbeta^{\dag}$.
The underling regression coefficient vector $\bbeta^{\dag} \in \mathbb{R}^{p}$ is generated in a way that $A^{\dag}$ is a randomly chosen subset of
$S$ with $|A^{\dag}|=T$. 
As in \cite{becker2011nesta}, \cite{shi2018admm} and \cite{shi2018semi},
each nonzero entry of $\bm{\beta}^{\dag}$ is generated as follows:
\begin{equation}\label{gendata}
\beta^\dag_j = \xi_{1j}10^{\xi_{2j}},
\end{equation}
where $j\in A^{\dag}$,
$\xi_{1j}=\pm 1$ with probability $\frac{1}{2}$ and
$\xi_{2j}$ is uniformly distributed in $[0,1]$.
Then the observation vector $\y =X\bbeta^{\dag}+\etaa$.
 For convenience, we use $(n,p,\rho,\sigma,T,\Rd)$ and
 $(n,p,\nu,\sigma,T,\Rd)$ to denote
the data generated as above, respectively.

\subsubsection{The behavior of the SNAP algorithm}

\paragraph{The algorithm parameters of SNAP}

We study the influence of the free parameters $N$ and $K$ in the
SNAP algorithm on the exact support recovery probability (Probability for short),
that is, the percentage of the estimated model $\Ah$ agrees with the true model $\Ad$.
To this end, we independently generate $20$ datasets from
$(n=200,p=1000,\rho=0.1,\sigma=0.01,T=5:5:30, \Rd=10)$  for each combination of $(N,K)$.
Here $5:5:30$ means the sparsity level starts from $5$ to $30$ with an increment of $5$.
The numerical results are summarized in \figref{fig:SNAP-NK},
which consider the following two settings:
(a) $K=1$, and varying $N \in \{40,60,80,100\}$;
(b) $N=100$, and varying $K \in \{1,2,3\}$.

\begin{figure}[ht]
\begin{subfigure}{.49\linewidth}
\centering
\includegraphics[width=\linewidth]{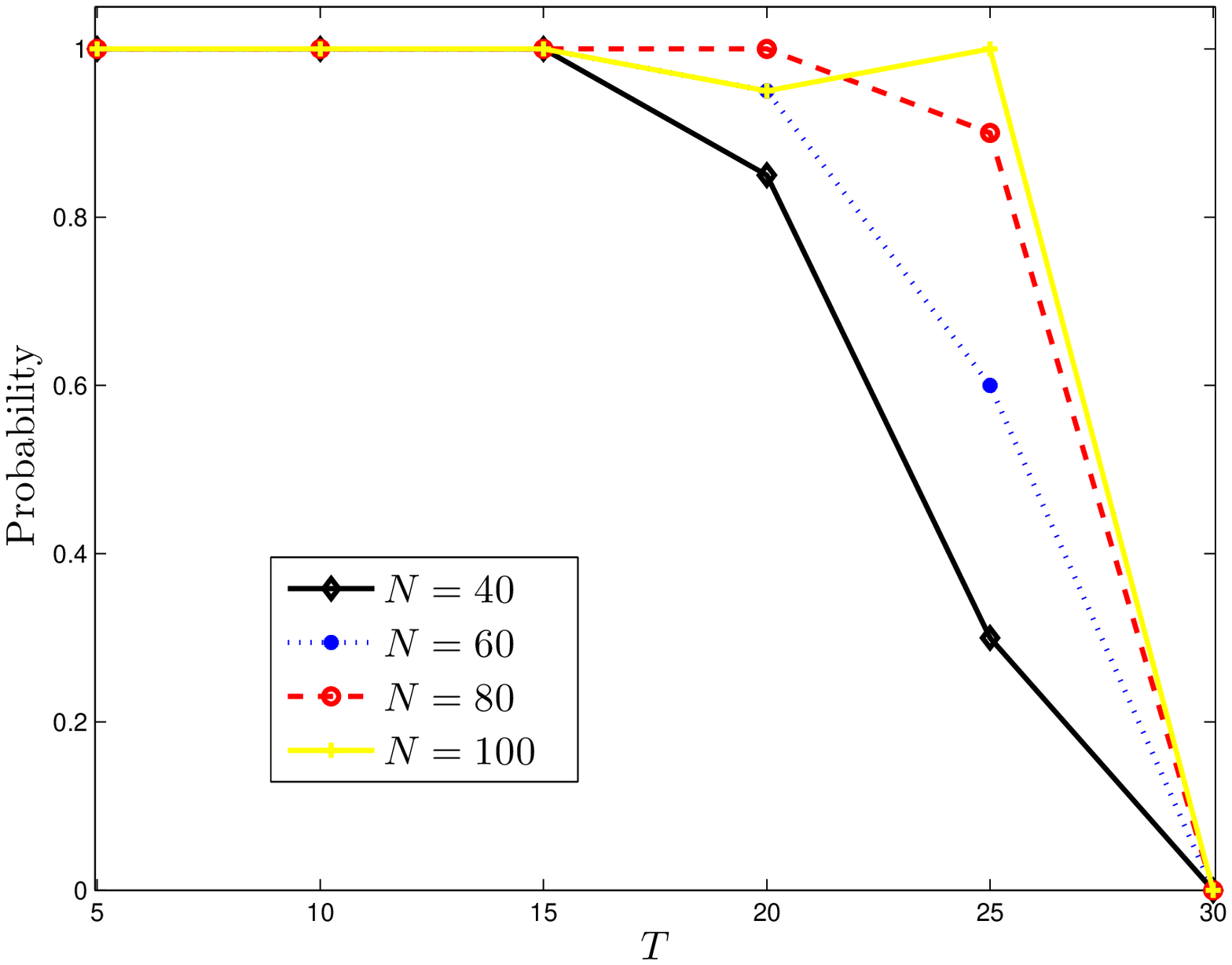}
\end{subfigure}
\begin{subfigure}{.49\linewidth}
\centering
\includegraphics[width=\linewidth]{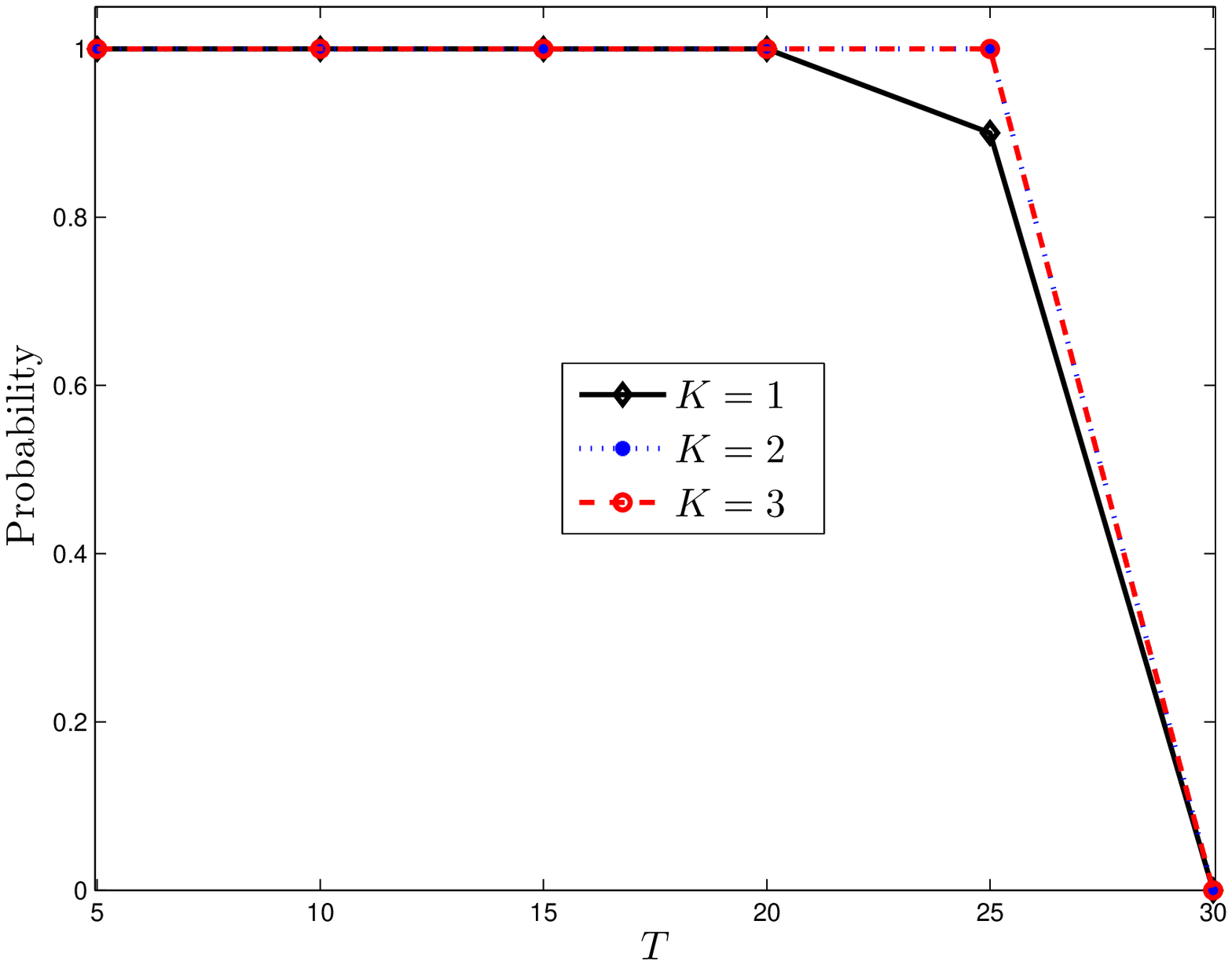}
\end{subfigure}
\caption{The influence of  the SNAP parameters $N$ (left panel) and
$K$ (right panel) on the exact support recovery probability.}
\label{fig:SNAP-NK}
\end{figure}

It is observed from \figref{fig:SNAP-NK} that
the influence of $K$ is very mild on the exact support recovery
probability and $K=1$  generally works well in practice,
due to the locally superlinear
convergence of SNA and the continuation technique with warm start
on the solution path, which is consistent with
the conclusions in \secref{sec:genSNAP}.
It is also found in \figref{fig:SNAP-NK} that
Larger $N$ values make the algorithm have better exact
support recovery probability, but the enhancement
decreases as $N$ increases.
Thus, unless otherwise specified, we set $(N,K)=(100,1)$ for the SNAP solver.

\paragraph{The MBIC selector for SNAP}

We illustrate the performance of the MBIC selector \eqref{mbic}
for SNAP with simulated data
$(n = 400, p = 2000, \rho = 0.5, \sigma = 0.1, T = 10, \Rd=10)$.
The results are summarized in  \figref{fig:SnapBehaveMbic}.
 It can be observed from \figref{fig:SnapBehaveMbic}
that the MBIC selector performs very well for the SNAP
algorithm on the continuation solution path introduced in \secref{ssec:solpath}.

\begin{figure}[!ht]
\begin{subfigure}{.33\linewidth}
\centering
\includegraphics[width=\linewidth]{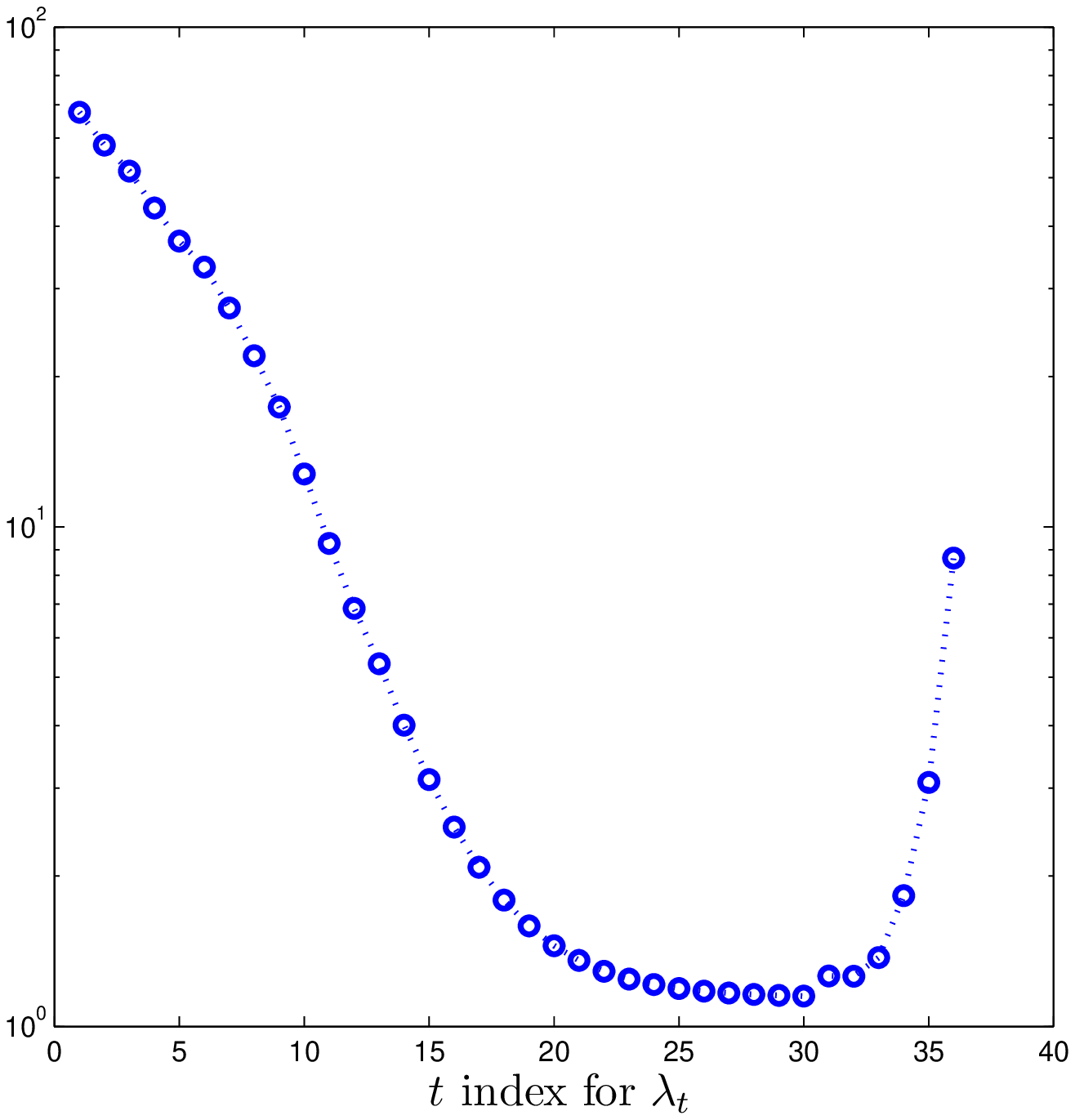}
\end{subfigure}
\begin{subfigure}{.33\linewidth}
\centering
\includegraphics[width=\linewidth]{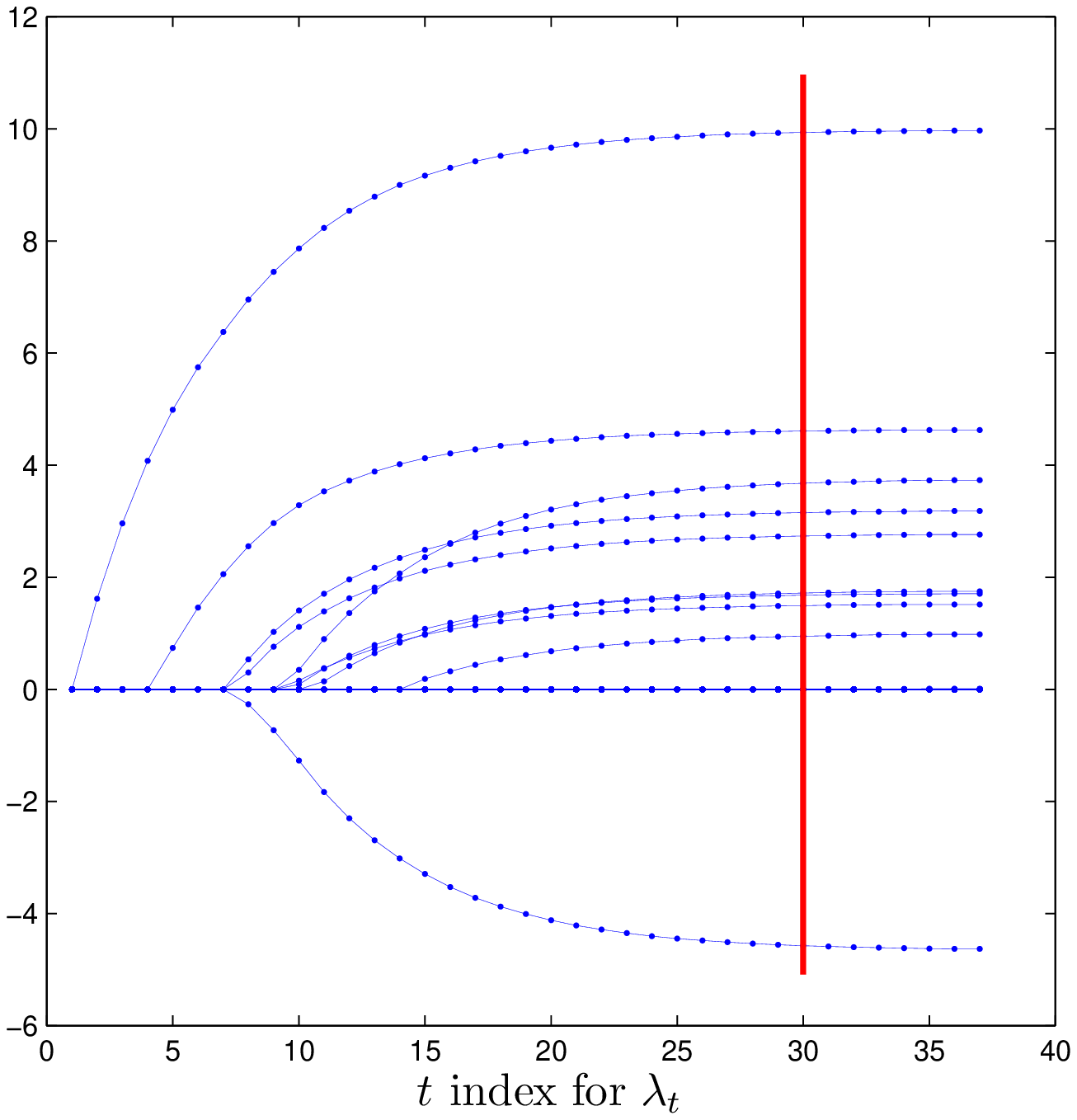}
\end{subfigure}
\begin{subfigure}{.33\linewidth}
\centering
\includegraphics[width=\linewidth]{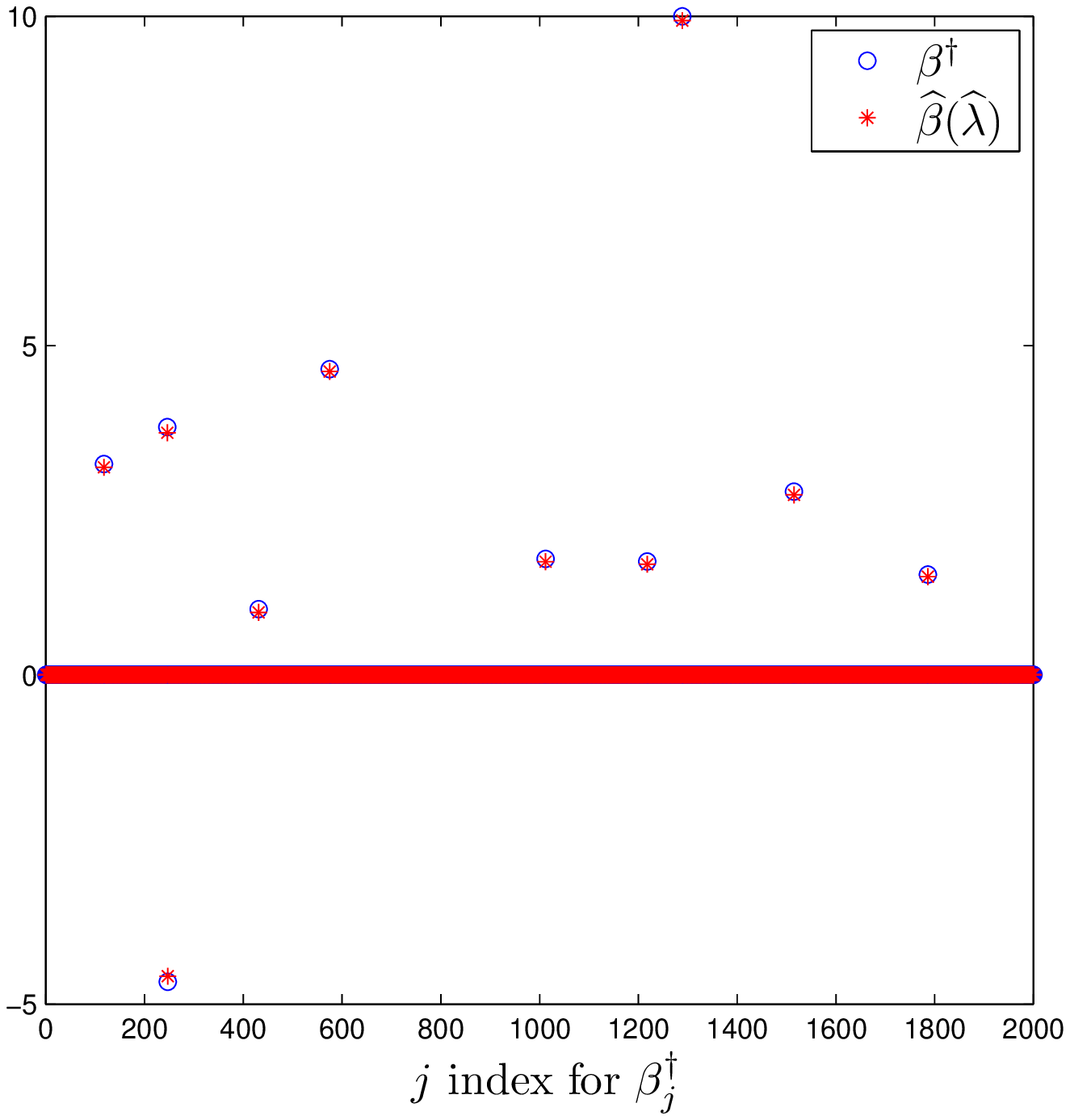}
\end{subfigure}
\caption{
Plots for  SNAP
using the MBIC  selector with data
$(n = 400, p = 2000, \rho = 0.5, \sigma = 0.1, T = 10, \Rd=10)$:
MBIC curve (left panel),
the solution path (middle panel),
and the comparison between
the underlying true parameter $\bbd$ and
the selected solution $\hbbeta(\hat{\la})$ (right panel).
The red vertical line in the middle panel
shows the solutions selected by MBIC.
}
\label{fig:SnapBehaveMbic}
\end{figure}

\paragraph{The local superlinear convergence of SNAP}

To gain further insight into the SNAP algorithm,
we illustrate the convergence behavior of the algorithm
using the simulated data as that of \figref{fig:SnapBehaveMbic}.
Let $\hat{A}_t=\{j:\;\hat{\beta}_j(\lambda_t)\neq 0\}$,
where $\hat{\bm{\beta}}(\lambda_t)$ is the solution to the $\lambda_t$-problem.
Set $(N,K)=(100,5)$.
The convergence history  is  shown  in \figref{fig:SnapBehaveConverge},
which presents the change of the active sets
and the number of iterations for each fixed $\lambda_t$
along the path $\lambda_0>\lambda_1>\cdots>\hat{\lambda}$.
It is observed in \figref{fig:SnapBehaveConverge} that
$\hat{A}_t \subset \Ad$, and  the size $|\hat{A}_t|$
increases monotonically as the path proceeds
and eventually equals the true model size $|\Ad|$.
In particular, for each $\lambda_{t+1}$ problem with $\hat{\bm{\beta}}(\lambda_t)$
as the initial guess,  SNAP generally reaches convergence
within two iterations (typically one, noting that
 the maximum number of iterations $K=5$ here).
This is attributed to the local superlinear/one step convergence of the algorithm
for LASSO, which is consistent with the results in \thmref{th5-6}.
Hence, the overall procedure is very efficient.

\begin{figure}[!ht]
\begin{subfigure}{.49\linewidth}
\centering
\includegraphics[width=\linewidth]{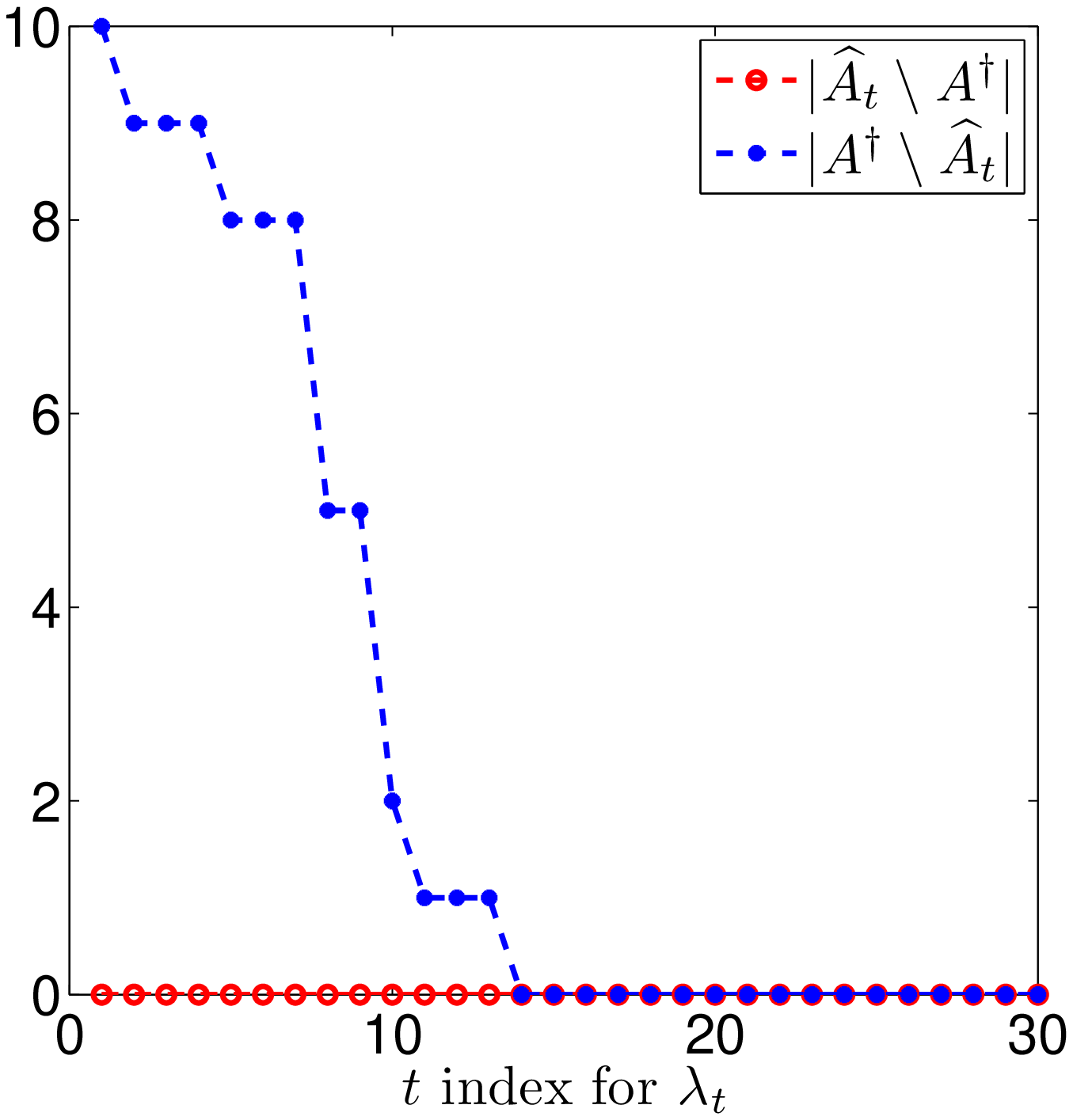}
\end{subfigure}
\begin{subfigure}{.49\linewidth}
\centering
\includegraphics[width=\linewidth]{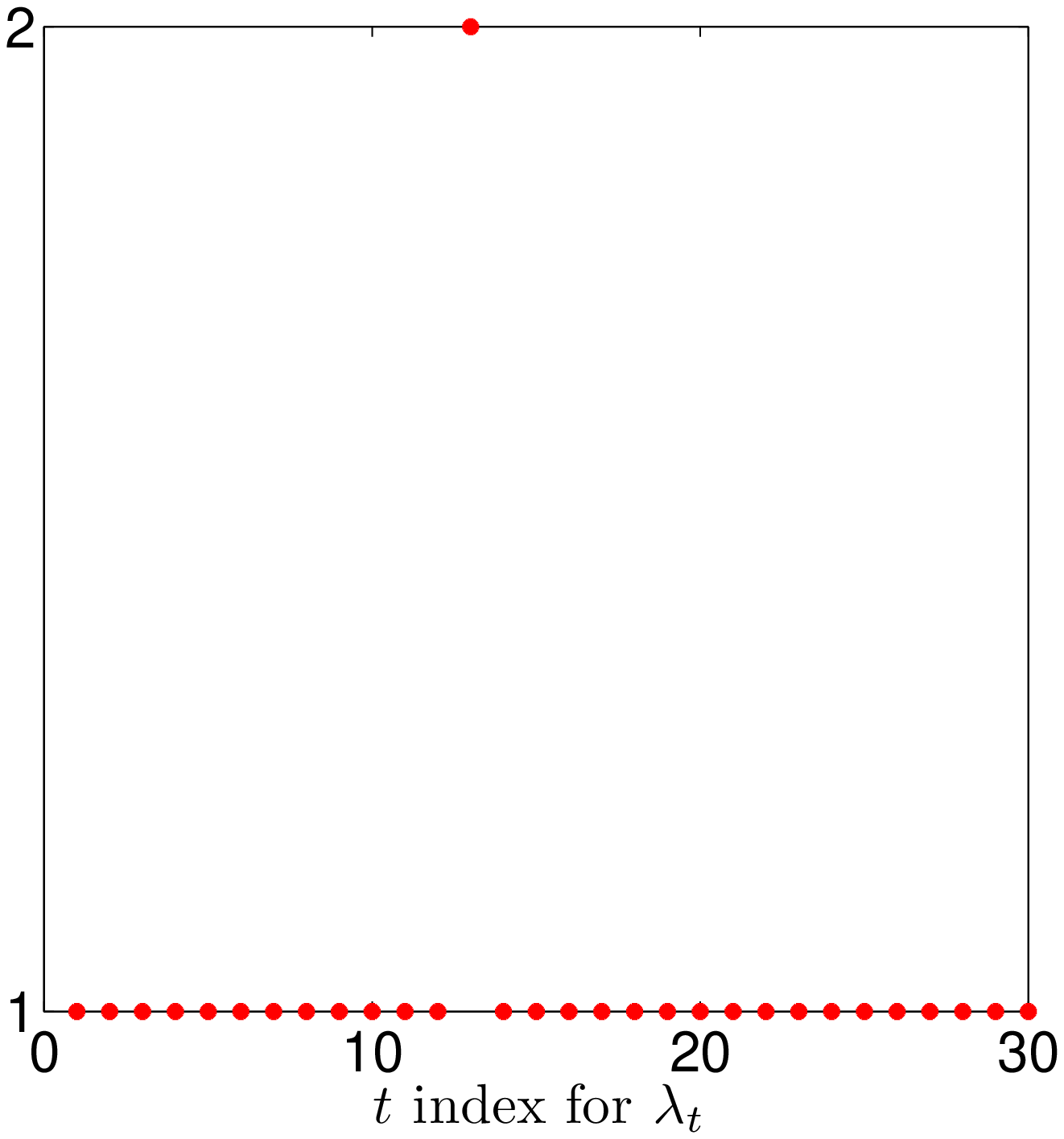}
\end{subfigure}
\caption{
Convergence behavior of SNAP with data
$(n = 400, p = 2000, \rho = 0.5, \sigma = 0.1, T = 10, \Rd=10)$:
the change of the active sets (left panel)
and the number of iterations (right panel)
for each $\lambda_t$-problem along the path.
$\hat{A}_t\backslash\Ad$ ($\Ad\backslash \hat{A}_t$)
denotes the set difference of sets $\hat{A}_t$
and $\Ad$ ($\Ad$ and $\hat{A}_t$).
In the left panel, the vertical axis is the size of sets;
in the right panel, the vertical axis is the number of iterations.
}
\label{fig:SnapBehaveConverge}
\end{figure}

\subsubsection{Efficiency and accuracy}

To evaluate the efficiency and accuracy of the proposed SNAP algorithm,
we independently generate $M=100$ datasets
from two settings:
(i) the classical Gaussian matrix with
$(n,p,\rho,\sigma,T,\Rd) = (600, 3000, 0.3:0.2:0.7, 0.2:0.2:0.4, 40, 10)$;
(ii) the  random Gaussian matrix with
$(n,p,\nu,\sigma,T,\Rd) = (1000, 10000, 0.3:0.2:0.7, 0.2:0.2:0.4, 50, 10)$.
Based on $M$ independent runs,
we compare SNAP with CD  and LARS
in terms  of  the average CPU  time (Time, in seconds),
the estimated average model size (MS)
$M^{-1}\sum_{m=1}^M \abs{\Ah^{(m)}}$,
the proportion of correct models (CM, in percentage terms)
$M^{-1}\sum_{m=1}^M I\dkhb{\Ah^{(m)}=\Ad}$,
the average $\li$ absolute error (AE)
$M^{-1}\sum_{m=1}^M \normi{\hbb^{(m)}-\bbd}$, and
the average $\lt$ relative error (RE)
$M^{-1}\sum_{m=1}^M \xkhb{\normt{\hbb^{(m)}-\bbd}/\normt{\bbd}}$.
The measure Time reflects the efficiency of the solvers,
while measures MS, CM, AE and RE evaluate the accuracy (quality) of the solutions.
Simulation results are summarized in \tabref{tab:simuSmall}
and \tabref{tab:simuBig}, respectively.

\begin{table}[!ht]
\caption{Simulation results for the classical Gaussian matrix
with $n=600$, $p=3000$, $T=40$ and $\Rd=10$ based on $100$ independent runs.
The numbers in the parentheses are the corresponding standard errors.}
\label{tab:simuSmall}
\centering
\scalebox{0.85}{
\begin{tabular}{cccccccc}
\hline
$\rho$&$\sigma$& Method&    Time&     MS&     CM&      AE&     RE\\	
\hline
  0.3&0.2 &  CD&0.2287(0.0060)&41.35(1.1135)&23\%(0.4230)&0.1270(0.0286)&0.0172(0.0042)\\
	 &    &LARS&0.2066(0.0336)&41.25(1.1315)&29\%(0.4560)&0.1208(0.0273)&0.0163(0.0039)\\
	 &    &SNAP&0.1784(0.0111)&40.07(0.2932)&94\%(0.2387)&0.0808(0.0192)&0.0105(0.0030)\\
	 &0.4 &  CD&0.2233(0.0023)&42.71(1.5973)& 8\%(0.2727)&0.2041(0.0349)&0.0275(0.0047)\\
	 &    &LARS&0.2121(0.0316)&42.14(3.0978)&13\%(0.3380)&0.2672(0.6802)&0.0344(0.0767)\\
	 &    &SNAP&0.1569(0.0041)&40.23(0.4894)&80\%(0.4020)&0.1528(0.0309)&0.0200(0.0046)\\
  0.5&0.2 &  CD&0.2272(0.0025)&42.74(1.8836)&11\%(0.3145)&0.1484(0.0454)&0.0193(0.0057)\\
	 &    &LARS&0.2086(0.0318)&42.53(2.0863)&15\%(0.3589)&0.1808(0.2411)&0.0223(0.0317)\\
	 &    &SNAP&0.1746(0.0036)&40.19(0.4861)&84\%(0.3685)&0.0925(0.0319)&0.0117(0.0039)\\
	 &0.4 &  CD&0.2236(0.0027)&44.48(2.1057)& 0\%(0.0000)&0.2333(0.0616)&0.0299(0.0066)\\
	 &    &LARS&0.2162(0.0364)&43.58(5.0835)& 1\%(0.1000)&0.3518(0.9235)&0.0427(0.1029)\\
	 &    &SNAP&0.1548(0.0039)&40.62(0.8138)&55\%(0.5000)&0.1693(0.0439)&0.0213(0.0045)\\
  0.7&0.2 &  CD&0.2280(0.0023)&47.96(3.2315)& 0\%(0.0000)&0.2062(0.0956)&0.0223(0.0061)\\
	 &    &LARS&0.2325(0.0317)&47.85(3.4855)& 0\%(0.0000)&0.3177(0.6256)&0.0255(0.0219)\\
	 &    &SNAP&0.1737(0.0047)&41.25(1.2340)&34\%(0.4761)&0.1274(0.0596)&0.0136(0.0040)\\
	 &0.4 &  CD&0.2249(0.0023)&50.59(3.6517)& 0\%(0.0000)&0.3323(0.1500)&0.0349(0.0081)\\
	 &    &LARS&0.2437(0.0310)&50.61(5.4028)& 0\%(0.0000)&0.5283(0.7663)&0.0463(0.0679)\\
	 &    &SNAP&0.1561(0.0040)&41.92(1.4885)&16\%(0.3685)&0.2337(0.0914)&0.0245(0.0054)\\
\hline
\end{tabular}
}
\end{table}

For each $(\rho, \sigma)$ combination,
it can be observed from \tabref{tab:simuSmall} that 
SNAP has  better speed performance than CD and LARS.
With $\rho$ fixed, the
CPU time of CD and SNAP slightly decreases as $\sigma$ increases,
while higher $\sigma$ increases the timing of LARS in general.
Given $\sigma$, the CPU time
of CD and SNAP is relatively robust with respect to $\rho$,
while that of LARS generally increases as $\rho$ increases.
According to MS, all solvers tend to overestimate the true model
and SNAP usually selects a smaller model,
while SNAP can select the correct model far more frequently than
CD and LARS in terms of CM.
The errors of all solvers AE and RE are small, which means
they all can produce estimates that are very close to the true values of $\bbd$,
while the AE and RE of SNAP are smaller than that of the other two,
indicating that SNAP is generally more accurate than CD and LARS.
Unsurprisingly, larger $\rho$ or $\sigma$ will degrade the accuracy of all solvers.
In addition, it is shown from \tabref{tab:simuSmall} that 
SNAP generally has smaller ((or comparable) standard errors, 
especially in accuracy  metrics MS, AE and RE, 
which means the results of SNAP are stable and robust.
Similar phenomena also hold for the random Gaussian matrix setting in \tabref{tab:simuBig}.
In particular, since the value of $p$ is large in \tabref{tab:simuBig}, 
the timing advantage of SNAP is more obvious, 
which implies that SNAP is capable of handling much larger data sets. 
In summary, SNAP behaves very well in simulation studies
and generally outperforms the state-of-the-art solvers such as LARS and CD
in terms of both efficiency and accuracy.

\begin{table}[!ht]
\caption{Simulation results for the random Gaussian matrix
with $n=1000$, $p=10000$, $T=50$ and $\Rd=10$ based on $100$ independent runs.
The numbers in the parentheses are the corresponding standard errors.}
\label{tab:simuBig}
\centering
\scalebox{0.85}{
\begin{tabular}{cccccccc}
\hline
$\nu$&$\sigma$& Method&    Time&     MS&     CM&      AE&     RE\\	
\hline
  0.3&0.2 &  CD&1.6607(0.0185)&51.59(1.5511)&26\%(0.4408)&0.0902(0.0229)&0.0121(0.0027)\\
	 &    &LARS&1.0458(0.1729)&51.48(1.5007)&29\%(0.4560)&0.0859(0.0215)&0.0116(0.0027)\\
	 &    &SNAP&0.8685(0.0289)&50.08(0.3075)&93\%(0.2564)&0.0553(0.0133)&0.0072(0.0017)\\
	 &0.4 &  CD&1.6416(0.0074)&52.76(1.8916)& 9\%(0.2876)&0.1403(0.0314)&0.0184(0.0031)\\
	 &    &LARS&1.1354(0.2245)&52.40(2.0792)&12\%(0.3266)&0.1560(0.2001)&0.0204(0.0272)\\
	 &    &SNAP&0.7764(0.0149)&50.28(0.5519)&76\%(0.4292)&0.1007(0.0217)&0.0128(0.0022)\\
  0.5&0.2 &  CD&1.6719(0.0058)&55.71(2.6678)& 0\%(0.0000)&0.0959(0.0583)&0.0111(0.0028)\\
	 &    &LARS&1.1819(0.2301)&55.59(3.3937)& 0\%(0.0000)&0.2002(0.4392)&0.0164(0.0377)\\
	 &    &SNAP&0.8980(0.0108)&50.82(1.0767)&49\%(0.5024)&0.0559(0.0320)&0.0064(0.0019)\\
	 &0.4 &  CD&1.6504(0.0064)&57.86(3.0847)& 0\%(0.0000)&0.1583(0.1068)&0.0164(0.0044)\\
	 &    &LARS&1.3180(0.2388)&56.85(4.2530)& 0\%(0.0000)&0.1954(0.4346)&0.0220(0.0584)\\
	 &    &SNAP&0.8098(0.0172)&51.20(1.1192)&31\%(0.4648)&0.1091(0.0674)&0.0112(0.0028)\\
  0.7&0.2 &  CD&1.6962(0.0100)&65.22(4.6113)& 0\%(0.0000)&0.1411(0.1521)&0.0114(0.0060)\\
	 &    &LARS&1.4585(0.2683)&64.90(7.2202)& 0\%(0.0000)&0.5015(1.0090)&0.0249(0.0411)\\
	 &    &SNAP&0.9439(0.0506)&53.09(1.7529)& 6\%(0.2387)&0.0858(0.1374)&0.0068(0.0100)\\
	 &0.4 &  CD&1.6715(0.0090)&67.12(4.8996)& 0\%(0.0000)&0.1880(0.2146)&0.0156(0.0082)\\
	 &    &LARS&1.5399(0.2448)&67.07(6.4499)& 0\%(0.0000)&0.5493(0.9422)&0.0271(0.0299)\\
	 &    &SNAP&0.8441(0.0889)&52.84(3.3536)& 6\%(0.2387)&0.1907(0.5365)&0.0187(0.0611)\\	
\hline
\end{tabular}
}
\end{table}


\subsection{Application}
We analyze the breast cancer data  which
comes from breast cancer tissue samples deposited to The Cancer Genome Atlas (TCGA) project and compiles results obtained using Agilent mRNA expression microarrays
to illustrate the application of the SNAP algorithm in high-dimensional settings.
This data, which is named bcTCGA,
is available at  \url{http://myweb.uiowa.edu/pbreheny/data/bcTCGA.RData}.
In this bcTCGA dataset,
we have expression measurements of 17814 genes from 536 patients
(all expression measurements are recorded on the log scale).
There are $491$ genes with missing data, which we have excluded.
We restrict our attention to the $17323$ genes without missing values.
The response variable $\y$ measures one of the $17323$ genes,
a numeric vector of length $536$ giving expression level of gene {BRCA1},
which is the first gene identified that increases the risk of
early onset breast cancer,
and the design matrix $X$ is a $536\times 17322$ matrix,
which represents the remaining expression measurements of 17322 genes.
Because {BRCA1} is likely to interact with many other genes,
it is of interest to find genes with expression levels
related to that of {BRCA1}.
This has been studied by using different
methods in the recent literature; see, for example,
\cite{tan2016bayesian,yi2017semismooth,lv2018oracle,breheny2018marginal,shi2018semi}.
In this subsection, we apply methods CD (glmnet), LARS (SolveLasso) and SNAP,
coupled with the HBIC selector, to analyze this dataset.

First, we analyze the complete dataset of 536 patients.
The genes selected by each method
along with their corresponding nonzero coefficient estimates,
the CPU time (Time, in seconds), the model size (MS)
and the prediction error (PE) calculated by
$n^{-1}\sum_{i=1}^n(\hat{y}_i-y_i)^2$ are provided in \tabref{tab:breast}.
It can be seen from \tabref{tab:breast} that
SNAP runs faster than LARS and CD,
while the PE by SNAP  is smaller than that by LARS and CD,
which demonstrates that SNAP  performs better than the other two solvers
in terms of both efficiency and accuracy.
Further, CD, LARS and SNAP identify 7, 9 and 4 genes respectively,
with 3 identified probes in common, namely, C17orf53, NBR2 and TIMELESS.
Although the magnitudes of estimates for the common genes are
not equal, they have the same signs, which suggests similar biological conclusions.

\begin{table}[!ht]
\caption{ The genes identified by CD, LARS and SNAP that correlated with BRCA1
based on the complete dataset of bcTCGA ($n=536, p=17322$).
The zero entries correspond to variables omitted.}
\label{tab:breast}
\centering
\begin{tabular}{cccccc}
\hline
No.	&   Term        &	   Gene&      CD&	 LARS&   SNAP\\	\hline
    &      Intercept&          &-1.0865 &-1.0217 &-0.4985\\
 1  &$\beta_{ 1743}$& C17orf53 & 0.1008 & 0.0983 & 0.4140\\
 2  &$\beta_{ 2739}$& CCDC56   & 0      & 0.0108 & 0     \\
 3  &$\beta_{ 2964}$& CDC25C   & 0      & 0.0136 & 0     \\
 4  &$\beta_{ 4543}$& DTL      & 0.0764 & 0.0844 & 0     \\
 5  &$\beta_{ 9230}$& MFGE8    & 0      & 0      &-0.1168\\
 6  &$\beta_{ 9941}$& NBR2     & 0.1519 & 0.1885 & 0.4673\\
 7  &$\beta_{12146}$& PSME3    & 0.0480 & 0.0615 & 0     \\
 8  &$\beta_{15122}$& TIMELESS & 0.0157 & 0.0279 & 0.2854\\
 9  &$\beta_{15535}$& TOP2A    & 0.0259 & 0.0331 & 0     \\
10  &$\beta_{16315}$& VPS25    & 0.1006 & 0.1083 & 0     \\	\hline
	&		    Time& 		   & 3.5436 & 2.1070 & 0.7884\\
	&		      MS& 		   & 7      & 9      & 4     \\
	&		      PE& 		   & 0.3298 & 0.3023 & 0.2345\\	\hline
\end{tabular}
\end{table}

To further evaluate the performance of the three methods,
we implement the cross validation (CV)  procedure similar to
\cite{huang2008adaptive,huang2010variable,tan2016bayesian,yi2017semismooth,
lv2018oracle,shi2018semi}.
We conduct $100$ random partitions of the data.
For each partition,  we randomly choose
$3/4$ observations and $1/4$ observations
as the training and test data, respectively.
We compute the CPU time (Time, in seconds) and the model size (MS, i.e.,
the number of selected genes) using the training data,
and calculate the prediction error (PE) based on the test data.
\tabref{tab:breastCv} presents the average values over $100$ random partitions,
along with corresponding standard deviations in the parentheses.

\begin{table}[!ht]
\caption{The CPU time (Time),  model size (MS) and prediction error (PE)
averaged across $100$ random partitions of the bcTCGA data
(numbers in parentheses are standard deviations)}
\label{tab:breastCv}
\centering
\scalebox{1}{
\begin{tabular}{cccc}
\hline
Method&           Time&           MS&             PE\\ \hline
    CD& 2.0598(0.0393)&	8.72(3.3667)& 0.3503(0.0764)\\
  LARS& 0.5861(0.3625)&	7.90(3.5689)& 0.3514(0.0831)\\
  SNAP& 0.6554(0.0906)&	6.10(3.0830)& 0.2742(0.0593)\\
\hline
\end{tabular}
}
\end{table}

Due to the CV procedure, the working sample size decreases to
$n_{CV}=\frac{3}{4}n$. Hence, the CPU time of three solvers
in \tabref{tab:breastCv} decrease accordingly compared with
the counterpart in \tabref{tab:breast}. Obviously,
it is shown in \tabref{tab:breastCv} that SNAP is still running faster than CD,
and is quite comparable to LARS in speed.
Compared to LARS, the CPU time of SNAP is less sensitive to the
sample size, which means SNAP has more potential than LARS
to be applied to a larger volume of noisy data.
Also, as clearly shown in the \tabref{tab:breastCv},
SNAP  selects  fewer genes and has a smaller PE,
which implies that SNAP could provide
a more targeted list of the gene sets.
Based on $100$ random partitions,
we report the selected genes and their corresponding frequency (Freq)
in \tabref{tab:breastCvFreq}, where the genes are ordered
such that the frequency is decreasing.
To save space, we only list genes with frequency greater than or equal to 5 counts.
It is observed from \tabref{tab:breastCvFreq} that
some genes such as NBR2, C17orf53, DTL and VPS25
have quite high frequencies (Freq $\geq 80$) with all three solvers,
which largely implies these genes are related to BRCA1.
Combining the findings in \tabref{tab:breastCvFreq} and
taking into account the small MS and PE of SNAP in \tabref{tab:breastCv},
we have a strong belief that genes NBR2 and C17orf53 selected by SNAP
are particularly associated with BRCA1.

\begin{table}[!ht]
\caption{Frequency table for $100$ random partitions of the bcTCGA data.
To save space, only the genes with Freq $\geq 5$ are listed.}
\label{tab:breastCvFreq}
\begin{center}
\scalebox{0.85}{
\begin{tabular}{llllllll} 
\hline
\multicolumn{2}{c}{CD}&&\multicolumn{2}{c}{LARS}&&\multicolumn{2}{c}{SNAP}\\
\cline{1-2} \cline{4-5} \cline{7-8}
Gene     &Freq&&     Gene&Freq&&     Gene&  Freq\\ \hline
C17orf53	&98	 &&C17orf53	&93	 &&NBR2		&95\\
DTL			&91  &&NBR2		&89  &&C17orf53	&91\\
NBR2		&91  &&VPS25	&87  &&DTL		&32\\
VPS25		&90  &&DTL		&86  &&MFGE8	&29\\
PSME3		&77  &&PSME3	&73  &&CCDC56	&25\\
TOP2A		&73  &&TOP2A	&67  &&CDC25C	&23\\
TIMELESS	&49  &&TIMELESS	&41  &&TUBG1	&23\\
CCDC56		&41  &&CCDC56	&33  &&LMNB1	&18\\
CDC25C		&35  &&CDC25C	&30  &&GNL1		&18\\
CENPK		&26  &&CENPK	&24  &&TIMELESS	&16\\
SPRY2		&20  &&RDM1		&20  &&VPS25	&16\\
SPAG5		&18  &&CDC6		&17  &&TOP2A	&15\\
RDM1		&18  &&TUBG1	&15  &&ZYX		&14\\
TUBG1		&17  &&SPRY2	&15  &&KIAA0101	&14\\
CDC6		&17  &&C16orf59	&12  &&KHDRBS1	&12\\
UHRF1		&16  &&CCDC43	&11  &&PSME3	&11\\
C16orf59	&13  &&UHRF1	&10  &&SPAG5	&10\\
CCDC43		&13  &&SPAG5	&10  &&TUBA1B	&8 \\
ZWINT		&9   &&NSF		&9   &&FGFRL1	&8 \\
KIAA0101	&9   &&KIAA0101	&8   &&CMTM5	&7 \\
NSF			&8   &&ZWINT	&5   &&SYNGR4	&5 \\
MLX			&6   &&         &    &&         &  \\
TRAIP		&5   &&         &    &&         &  \\
\hline
\end{tabular}
}
\end{center}
\end{table}

\section{Concluding Remarks}

Starting from the KKT conditions  we developed SNA for computing
the LASSO and Enet solutions in high-dimensional linear regression models.
We approximate the whole solution paths using SNAP by utilizing  the continuation technique
with warm start. SNAP is easy to implement, stable, fast and accurate.
We established the locally superlinear  of SNA for the
Enet and local one-step convergence for the LASSO.
We provided sufficient conditions under which  SANP enjoys the sign consistency property
in finite steps. Moreover, SNAP has the same computational complexity as LARS and CD.  Our simulation studies demonstrate that SNAP is competitive with these state-of-the-art solvers in accuracy and outperforms them in efficiency.
These theoretical and numerical results suggest that SNAP is a promising new method for dealing with large-scale 
$\ell_1$-regularized linear regression problems.

We have only considered the linear regression model with convex penalties.
It would be interesting to generalize SNAP to other models such as the
generalized linear 
and Cox models.
It would also be interesting to extend the idea of SNAP to problems
with nonconvex penalties such as SCAD \citee{fan2001variable}
and MCP \citee{zhang2010nearly}. Coordinate descent algorithms for these penalties
have been considered by
\cite{breheny2011coordinate} and \cite{mazumder2011sparsenet}.
In our paper we adopt simple continuation strategy to globalize SNA, globalization via smoothing  Newton methods \citee{chen1998global,qi1999survey,qi2000new}
is  also an interesting future work.

We have implemented SNAP in a Matlab package \textsl{snap}, which is available
at \url{http://faculty.zuel.edu.cn/tjyjxxy/jyl/list.htm}.

\section*{Acknowledgment}
\addcontentsline{toc}{section}{Acknowledgment}

The authors sincerely thank
Prof. Defeng Sun and Prof. Bangti Jin
for their helpful personal communications and suggestions.

Yuling Jiao is  supported by
the National Natural Science Foundation of China
(Grant Nos. 11501579 and 11871474),
Xiliang Lu is supported by
the National Natural Science Foundation of China
(Grant No. 11471253),
Yueyong Shi is supported by
the National Natural Science Foundation of China
(Grant Nos. 11501578, 11701571, 11801531 and 41572315),
and
Qinglong Yang is supported by
the National Natural Science Foundation of China
(Grant No. 11671311).

\section*{Appendices}
\addcontentsline{toc}{section}{Appendices}
\renewcommand{\theequation}{A.\arabic{equation}}
\renewcommand{\thesubsection}{\Alph{subsection}~}

\subsection{Background on  convex analysis and Newton derivative}

In order to derive the KKT system (\ref{K1})  and prove
the locally superlinear convergence of Algorithm \ref{alg1},
we recall some  background in convex analysis \citee{rockafellar1970convex}
 and  describe the concept and some properties of
Newton derivative \citee{kummer1988newton,qi1993nonsmooth,ito2008lagrange}.

The standard Euclidean inner product for two vector $\z,\w\in
\mathbb{R}^{p}$ is defined by
$\langle
\z,\w\rangle:=\sum_{i=1}^{p}z_{i}w_{i}$. The class of all proper lower
semicontinuous convex functions on $\mathbb{R}^{p}$ is denoted by
$\Gamma_{0}(\mathbb{R}^{p})$. The subdifferential of
$f:\mathbb{R}^{p}\rightarrow \mathbb{R}^{1}$ denoted by $\partial f$
is a set-value mapping defined as

\begin{equation*}
 \partial f (z):= \{\w\in \mathbb{R}^{p}:f(\v)\geq f(\z)+ \langle \w,\v-\z\rangle, \textrm{for all} \quad \v \in \mathbb{R}^{p}\}.
\end{equation*}
If $f$ is convex and differentiable it holds that
 \begin{equation}\label{sub1}
 \partial f (\z) =  \nabla f(\z)
 \end{equation}
Furthermore, if $f, g \in \Gamma_{0}(\mathbb{R}^{p})$ then
\begin{equation}\label{sub2}
 \partial (f+g)(\z) =   \partial f(\z)+ \partial g(\z)
 \end{equation}
Recall the classical Fermat's rule  \citee{rockafellar1970convex},
\begin{equation}\label{fermat}
\textbf{0} \in \partial f(\z^{*}) \Leftrightarrow \z^{*} \in \mathop \text{argmin}\limits_{\z\in \mathbb{R}^{p}}
\ f(z).
\end{equation}
Moreover, a more general case is  \citee{combettes2005signal}
\begin{equation}\label{general}
\w \in  \partial   f(\z) \Leftrightarrow \z= \textrm{Prox}_{f}(\z+\w),
\end{equation}
where $ \textrm{Prox}_{f}$ is the proximal operator for  $f \in
\Gamma_{0}(\mathbb{R}^{p}$)   defined as
\begin{equation*}
\textrm{Prox}_{f}(\z) := \mathop \text{argmin}\limits_{\x\in \mathbb{R}^{p}}
\frac{1}{2}\|{\x-\z}\|_{2}^2 + f(\x).
\end{equation*}
Here we should mention that the proximal operator of $\lambda\|\cdot\|_{1}$  is given in a closed form by the componentwise soft-threshold operator, i.e.,
\begin{equation}\label{proxl1}
\textrm{Prox}_{\lambda \| \x\|_{1}}(z) = T_{\lambda }(\x),
\end{equation}
where $T_{\lambda}(\x)$ is defined in (\ref{softth}).

Let $F:\mathbb{R}^{m}\rightarrow \mathbb{R}^{l}$ be a nonlinear
map. \cite{chen2000smoothing}
generalized classical  Newton's algorithm for finding a root of $F(\z)=\textbf{0}$ when $F$ is not Fr\'{e}chet differentiable but only Newton differentiable \citee{ito2008lagrange}.
\begin {definition}
$F:\mathbb{R}^{m}\rightarrow \mathbb{R}^{l}$ is called Newton
differentiable at $\x\in \mathbb{R}^{m}$ if there exists an open
neighborhood $N(\x)$ and a family of mappings $D : N(x)\rightarrow
\mathbb{R}^{l\times m}$ such that
$$ \| F(\x+\h)-F(\x)-D(\x+\h)\h\|_{2}= o (\|\h\|_{2}) \quad \text{for} \quad \|\h\|_{2} \longrightarrow0.$$
The set of maps $\{ D(\z):\z\in N(\x)\}$ denoted  by $\nabla_{N}F(\x)$
is called the  Newton derivative of $F$ at $\x$.
\end {definition}

It can be easily seen that $\nabla_{N}F(\x)$ coincides with the
Fr\'{e}chet derivative at $\x$ if $F$ is continuously Fr\'{e}chet
differentiable.  An example that is Newton differentiable but not
Fr\'{e}chet differentiable  is the absolute function $F(z) = |z|$
defined on $\mathbb{R}^{1}$. In fact, let $G(z+h)h = \frac{z+h}{|z+h|}h$ and $G(0)h = rh$ with $r$ be any constant in  $\mathbb{R}^{1}$. Then
\begin{equation}\label{nd1}
 \nabla_{N}F(z) =
  \left\{
    \begin{array}{ll}
   $ 1$,    \quad &\text{$z>0,$}\\
   $ -1$  ,  \quad &\text{$z<0,$}\\
   $\text{$ r$ }$ \in \mathbb{R}^{1},  \quad &\text{$z  = 0.$}
    \end{array}
  \right.
\end{equation}
follows from the definition of Newton derivative.

Suppose $F_{i}:\mathbb{R}^{m}\rightarrow\mathbb{R}^{1}$ is Newton
differentiable at $\x$ with Newton derivative $\nabla_{N}F_{i}(\x)$,
$i = 1,\ldots,l$.  Then
$F=(F_1, \ldots, F_l)'$ is also Newton
differentiable at $\x$ with Newton derivative
\begin{equation}\label{nd2}
\nabla_{N}F(\x)=\left(
                             \begin{array}{c}
                              \nabla_{N}F_{1}(\x) \\
                               \nabla_{N}F_{2}(\x) \\
                               \vdots \\
                               \nabla_{N}F_{l}(\x)\\
                             \end{array}
                           \right).
\end{equation}
Furthermore, if $F_{1}$ and $F_{2}$  are Newton differentiable at $\x$, then the linear combination of them are also Newton differentiable at $\x$, i.e., for any $\theta,\gamma \in \mathbb{R}^{1}$,
\begin{equation}\label{nd3}
 \nabla_{N}(\theta F_{1} +\gamma  F_{2} )(\x) =   \theta \nabla_{N}F_{1} (\x) + \gamma  \nabla_{N} F_{2} (\x).
 \end{equation}

Let $F_{1}:\mathbb{R}^{s}\rightarrow \mathbb{R}^{l}$ be Newton differentiable with Newton derivative $\nabla_{N}F_{1}$. Let $L\in \mathbb{R}^{s\times m}$ and define $F(\x)=F_{1}(L\x+\z)$. It can be verified that the chain rule holds, i.e., $F(\x)$ is Newton differentiable at $\x$ with Newton derivative
\begin{equation}\label{nd4}
\nabla_{N}F(\x) = \nabla_{N}F_{1}(L\x+\z)L.
\end{equation}

With the above preparation we can  calculate the Newton derivative of the componentwise soft threshold operator $T_{\lambda}(\x)$.
\begin{lemma}
\label{Lem1}
 $T_{\lambda}(\cdot):\mathbb{R}^{p}\rightarrow\mathbb{R}^{p}$ is Newton
 differentiable at any point $\x\in \mathbb{R}^{p}$. And ${\diag}(\bfb) \in \nabla_{N}T_{\lambda}(\x)$,
 where  ${\diag}(\bfb)$ is a diagonal  matrix with
 \[
\bfb=(\textbf{1}_{\{|x_{1}|>\lambda\}}, \ldots, \textbf{1}_{\{|x_{p}|>\lambda\}})^{\prime},
 \]
and  $\textbf{1}_A$ is the indicator function of set $A$.
\end{lemma}
This lemma is used in the derivation of the SNA given in Subsection
\ref{SNA-subsection}.

\noindent
\textbf{Proof of Lemma \ref{Lem1}.}  As shown in (\ref{nd1}), $\textbf{1}_{\{|z|>0\}}\in \nabla_{N}|z|$. Then, it follows from (\ref{nd3})-(\ref{nd4}) that the
scalar  function $T_{\lambda}(z)= z -
|z+\lambda|/2+|z-\lambda|/2$ is Newton differentiable by
 with
\begin{equation}\label{pa1}
\textbf{1}_{\{|z|>\lambda\}}\in \nabla_{N} T_{\lambda}(z).
\end{equation}
 Let $$F_{i}(\x)= T_{\lambda}(e_{i}^{\prime}\x): \x\in
\mathbb{R}^{p}\rightarrow
\mathbb{R}^{1}, i= 1,\ldots, p,$$
where the column vector $e_{i}$ is the $i_{th}$
orthonormal  basis in $\mathbb{R}^{p}$. Then, it follow from (\ref{nd4}) and (\ref{pa1}) that
\begin{equation}\label{pa2}
e_{i}^{\prime}\textbf{1}_{\{|x_{i}|>\lambda\}}\in\nabla_{N}F_{i}(\x).
\end{equation}
By using   (\ref{nd2})
and (\ref{pa2}) we have
 $T_{\lambda}(\x)=(F_1(\x), \ldots, F_p(\x))'$
is Newton differentiable and $
\diag\{\bfb\} \in \nabla_{N}T_{\lambda}(\x)$.
This completes the proof of Lemma \ref{Lem1}.

\subsection{Proofs}

\medskip\noindent
\textbf{Proof of Proposition  \ref{th1}.}
\begin{proof}
This is a standard result in convex optimization, we include a proof here for completeness.
 Obviously  $L_{\lambda}(\cdot)$ is bounded below by $0$, thus, has
 infimum denoted by $L^{*}$. Let $\{\bbeta^{k}\}_{k}$ be a
 sequence such that $L_{\lambda}(\bbeta^{k})\rightarrow
 L^{*}$.
    Then $\{\bbeta^{k}\}_{k}$ is bounded due to
\begin{equation}\label{pb1}
 L_{\lambda}(\bbeta)\rightarrow
 +\infty  \quad \text{as}   \quad\|\bbeta\|_{1}\rightarrow
 +\infty.
 \end{equation}
Hence $\{\bbeta^{k}\}_{k}$ has a subsequence still denoted by
$\{\bbeta^{k}\}_{k}$ that converge to some $\bbeta_{\lambda}$. Then the continuity of $L_{\lambda}(\cdot)$
implies $\bbeta_{\lambda}\in M_{\lambda}$, i.e., $M_{\lambda}$ is nonempty.
The boundedness of $M_{\lambda}$  follows from (\ref{pb1}) and the
closeness follows from the continuity of $L_{\lambda}(\cdot)$, i.e.,
$M_{\lambda}$ is compact. The convexity of $M_{\lambda}$ follows from the convexity of $L_{\lambda}(\cdot)$.
This completes the proof of Proposition \ref{th1}.
\end{proof}

\medskip\noindent
\textbf{Proof of Proposition \ref{pr2}}
\begin{proof}
By the same argument in the proof of Proposition  \ref{th1}, there exists a
minimizer of $J_{\lambda,\alpha}(\cdot)$. We denote this minimizer by $\hbbeta_{\la,\alpha}$.
It follow from the strict convexity of
$J_{\lambda,\alpha}(\cdot)$  that  $\hbbeta_{\la,\alpha}$ is unique.
Let $\hbbeta_{\lambda}$ be the  one in $M_{\lambda}$ with the minimum Euclidean.
We have
 \begin{align}\label{pc1}
 L_{\lambda}(\hbbeta_{\lambda}) +
 \frac{\alpha}{2n}\|\hbbeta_{\la,\alpha}\|^{2}_{2} &\leq
 L_{\lambda}(\hbbeta_{\la,\alpha})+\frac{\alpha}{2n}\|\hbbeta_{\la,\alpha}\|^{2}_{2}=
 J_{\lambda,\alpha}(\hbbeta_{\la,\alpha})\notag\\
 &\leq
J_{\lambda,\alpha}(\hbbeta_{\lambda})=
 L_{\lambda}(\hbbeta_{\lambda})+\frac{\alpha}{2n}\|\bbeta_{\lambda}\|^{2}_{2},
\end{align}
 where the first inequality use the the property that $\hbbeta_{\lambda}$  is a minimizer of $L_{\lambda}(\cdot)$,
 and the second inequality use the the property that $\hbbeta_{\la,\alpha}$  is a minimizer of $J_{\lambda,\alpha}(\cdot)$.
  Then it follows from (\ref{pc1}) that
  \begin{equation}\label{pc2}
 \|\hbbeta_{\la,\alpha}\|^{2}_{2} \leq
 \|\hbbeta_{\lambda}\|^{2}_{2}.
 \end{equation}
This implies $\{\hbbeta_{\la,\alpha}\}_{\alpha}$ is bounded and thus  there exist a subsequence of $\{\hbbeta_{\la,\alpha}\}_{\alpha}$
denoted by $\{\bar{\bbeta}_{\la,\alpha}\}_{\alpha}$  that converge to some
 $\bbeta_{*}$ as $\alpha \rightarrow 0^{+}$.  Let  $\alpha \rightarrow
 0^{+}$  in (\ref{pc1}) and (\ref{pc2}) we get
$$L_{\lambda}(\bbeta_{*})\leq
L_{\lambda}(\hbbeta_{\lambda})$$ and
$$  \|\bbeta_{*}\|_{2} \leq
 \|\hbbeta_{\lambda}\|_{2}.$$
The above two inequality imply
$\bbeta_{*}$ is a minimizer of $L_{\lambda}(\cdot)$ with
minimum 2-norm. Thus,  $\bbeta_{*} = \hbbeta_{\lambda}$
due to the uniqueness of such a minimizer.
Hence $\bar{\bbeta}_{\la,\alpha}$ converges to
$\bbeta_{\lambda}$.  The same argument shows that any
subsequence of $\{\hbbeta_{\la,\alpha}\}_{\alpha}$  has a further subsequence
converging to $\bbeta_{\lambda}$.  This implies that the whole sequence
$\{\hbbeta_{\la,\alpha}\}_{\alpha}$ converges to $\hbbeta_{\lambda}$.
This completes the proof of Proposition  \ref{pr2}.
\end{proof}

\medskip\noindent
\textbf{Proof of Proposition \ref{th3}.}
\begin{proof} We first assume $\hbbeta_{\la,\alpha} \in \Rp$ is a minimizer of
(\ref{regLASSO*}).  Then it follows from  (\ref{sub1})-(\ref{fermat}) that
$$\textbf{0}  \in X^{\prime}(X \hbbeta_{\la,\alpha} - \y)/n + \alpha \hbbeta_{\la,\alpha}+\lambda\partial\|\cdot\|_{1}(\hbbeta_{\la,\alpha}).$$
Therefore, there exists $\hbd_{\la,\alpha} \in
\lambda\partial\|\cdot\|_{1}(\hbbeta_{\la,\alpha})$
such that
$$\textbf{0}  =
 X^{\prime}(X \hbbeta_{\la,\alpha} - \y)/n +\alpha \hbbeta_{\la,\alpha}+ \hbd_{\la,\alpha},$$ i.e. the first equation of (\ref{K1}) holds by noticing $G=X^{\prime}X+\alpha I$ and $\tby = X^{\prime}\y $.
Furthermore,  it follow from (\ref{general}) that
 $$\hbd_{\la,\alpha} \in \lambda\partial\|\cdot\|_{1}( \hbbeta_{\la,\alpha})$$
 is equivalent to
 $$\hbbeta_{\la,\alpha} =
\textrm{Prox}_{\lambda\partial\|\cdot\|_{1}}( \hbbeta_{\la,\alpha}+\hbd_{\la,\alpha}).
$$
 By using  (\ref{proxl1}), we have
 $$\hbbeta_{\la,\alpha} = T_{\lambda}(\hbbeta_{\la,\alpha} +
\hbd_{\la,\alpha}),$$ which is the second  equation of (\ref{K1}).

Conversely, if (\ref{K1}) are satisfied for some
$\hbbeta_{\la,\alpha} \in \Rp$, $\hbd_{\la,\alpha} \in \Rp$. By using (\ref{general}) and
(\ref{proxl1}) again, we deduce
$$\hbd_{\la,\alpha} \in \lambda\|\cdot\|_{1}(\hbbeta_{\la,\alpha})$$
from the second  equation of (\ref{K1}).
Substituting this into the first equation of (\ref{K1}) we have
$$\textbf{0} \in (G \hbbeta_{\la,\alpha} -\tby)/n +
 \lambda\|\cdot\|_{1}(\hbbeta_{\la,\alpha}),$$
which implies  that
$\hbbeta_{\la,\alpha}$ is a minimizer of (\ref{regLASSO*}) by Fermat's rule
(\ref{fermat}).

The proof for (\ref{regLASSO}) can be derived similarly.
This completes the proof of Proposition \ref{th3}.
\end{proof}

\medskip\noindent
\textbf{Proof of Theorem \ref{th4}.}

\begin{proof}
It follows from Lemma \ref{Lem1} and  (\ref{nd3})-(\ref{nd4}) that
$ F_{1}(\z)$ is  Newton differentiable. Furthermore, by using  Lemma \ref{Lem1} and the definition of $A$ and $B$, we have
\begin{equation}\label{pe1}
\begin{bmatrix}
& -I_{AA}    & \textbf{0}                                                &\textbf{0}                                             & \textbf{0} \\ \\
& \textbf{0}          & I_{BB} &\textbf{0} &\textbf{0}
\end{bmatrix}
\in \nabla_{N}F_{1}(\z).
\end{equation}

Obviously, $F_{2}(\z)$ is continuously differentiable with
 \begin{equation}\label{pe2}
  \nabla F_{2}(\z)=
\begin{bmatrix}
& nI_{AA}    & X_{A}^{\prime}X_{B}                                   &G_{AA}             & \textbf{0} \\ \\
& \textbf{0}          & G_{BB}            &X_{B}^{\prime}X_{A}      &
nI_{BB}
\end{bmatrix}.
\end{equation}
Then it follows from (\ref{pe1})-(\ref{pe2}) and (\ref{nd2}) that
 $F$ is Newton differentiable  $\z$  with $H \in
\nabla_{N}F(\z)$.

Let
\[
H_{1}=
\begin{bmatrix}
& -I_{AA}    & \textbf{0}\\
 & \textbf{0}          & I_{BB}
\end{bmatrix},
 H_{2}=
\begin{bmatrix}
& nI_{AA}    & X_{A}^{\prime}X_{B}\\\\
 & \textbf{0}          & G_{BB}
\end{bmatrix},
H_{3}=
\begin{bmatrix}
&G_{AA}             & \textbf{0} \\ \\
&X_{B}^{\prime}X_{A}      & nI_{BB}
\end{bmatrix}.
\]
Obviously, $H_{i}, i=1,2,3$ is invertible and
\[
H^{-1} =
\begin{bmatrix}
& H_{1}^{-1}   & \textbf{0}\\
 & -H_{3}^{-1}H_{2}H_{1}^{-1}          & H_{3}^{-1}
\end{bmatrix}.
\]
Let $g=(g_1',g_2')'$
be an arbitrary vector in $\mathbb{R}^{2p}$. Then
$$
  \|H^{-1}g\|_{2}^{2}=  \|\begin{bmatrix}
& H_{1}^{-1}   & \textbf{0}\\
 & -H_{3}^{-1}H_{2}H_{1}^{-1}          & H_{3}^{-1}
\end{bmatrix}
 \left(
      \begin{array}{c}
        g_{1} \\
        g_{2} \\
      \end{array}
    \right)\|_{2}^{2}
 $$
$$ \quad  \quad \quad \quad \quad \quad = \|H_{1}^{-1}g_{1}\|_{2}^{2} +  \|-H_{3}^{-1}H_{2}H_{1}^{-1}g_{1}+ H_{3}^{-1}g_{2}\|_{2}^{2}$$
$$ \quad  \quad \quad \quad \quad \quad\leq \|H_{1}^{-1}\|\|g_{1}\|_{2}^{2} +  \|H_{3}^{-1}\|^{2}(\| H_{2}\|\|H_{1}^{-1}\|\| g_{1}\|_{2}+ \|g_{2}\|_{2})^2$$
$$ \quad  \quad \leq  (\|H_{1}^{-1}\| + \|H_{3}^{-1}\|(1+\| H_{2}\|\|H_{1}^{-1}\|))^2\|
g\|_{2}^{2},$$ which shows
 \begin{equation}\label{pe3}
\|H^{-1}\| \leq \|H_{1}^{-1}\| +
\|H_{3}^{-1}\|(1+\| H_{2}\|\|H_{1}^{-1}\|).
\end{equation}
The similar argument shows
\begin{equation}\label{pe4}
\|H_{2}\| \leq n + \alpha + 2 \norm X ^2,
\end{equation}
and
\begin{equation}\label{pe5}
\|H_{3}^{-1}\| \leq 1/n + (1+\norm X ^2)/\alpha.
\end{equation}
Combining   (\ref{pe4})-(\ref{pe5}) with  (\ref{pe3}) and observing $\|H_{1}^{-1}\| =1$  we get
$$\|H^{-1}\|< 1+2(n+1+\alpha+\norm X^2)^2/\alpha.$$
This completes the proof of Theorem \ref{th4}.
\end{proof}

\medskip\noindent
\textbf{Proof of Theorem \ref{th5}.}
\begin{proof}
 Let $\z_{\alpha}=(\hbbeta_{\alpha}^{\prime},\hbd_{\alpha}^{\prime})^{\prime}$
be a root of $F(\z)$.  Let $\z^{k}$ be sufficiently close to $\z_{\alpha}$.  By using the definition of Newton derivative and $H_{k}\in \nabla_{N}F(\z^{k})$, we have
\begin{equation}\label{pth5}
\|H_{k}(\z^{k}-\z_{\alpha})-F(\z^{k}) + F(\z_{\alpha})\|_{2}\leq  \vps \|\z_{k}-\z_{\alpha}\|_{2},
\end{equation}
 where $\vps\rightarrow 0$ as  $\z^{k}\rightarrow \z_{\alpha}$.
Then,
\begin{align*}
&\|{\z^{k+1}-\z_{\alpha}}\|_{2} \\
&= \|\z^{k}-H_{k}^{-1}F(\z^{k})-\z_{\alpha}\|_{2} \\
&=\|\z^{k}-H_{k}^{-1}F(\z^{k})-\z_{\alpha} + H_{k}^{-1}F(\z_{\alpha})\|_{2}\\
&\leq \norm{H_{k}^{-1}} \|H_{k}(\z^{k}-\z_{\alpha})-F(\z^{k}) + F(\z_{\alpha})\|_{2} \\
&\leq \vps (1+2(n+1+\alpha+\norm X^2)^2/\alpha)  \|\z^{k}-\z_{\alpha}\|_{2},
\end{align*}
where the first equality uses  (\ref{ssnd}) - (\ref{ssnup}), the second equality uses
 $F(\z_{\alpha})=\textbf{0}$,  the first inequality is some algebra, and the last inequality uses
(\ref{pth5})
and  the uniform boundedness of $H_{k}^{-1}$ proved in  Theorem \ref{th4}.
Then we get the sequence  $\z^{k}$ generated by Algorithm \ref{alg3} converge
to $\z^{\alpha}$ locally superlinearly. The definition of $F(\z)$  implies its   root $\z_{\alpha}= (\hbbeta_{\alpha}^{\prime},\hbd_{\alpha}^{\prime})^{\prime}$
satisfies the KKT conditions (\ref{K1}).
 Thus, it follows from  Proposition \ref{th3} that  $\hbbeta_{\alpha}$ is the unique minimizer of (\ref{regLASSO*}).  Therefore, Theorem \ref{th5} holds by the equivalence between Algorithm \ref{alg1} and Algorithm \ref{alg3}.
This completes the proof of Theorem \ref{th5}.
\end{proof}

\medskip\noindent
\textbf{Proof of Theorem \ref{th5-6}.}
\begin{proof}
First, we have
\begin{align*}
&\hbeta_{i}+\hd_{i} -\beta_{i}^0-d_{i}^0 \\&\leq
|\beta_{i}^0+d_{i}^0 -\hbeta_{i}-\hd_{i} |\\&\leq \|\hbbeta - \bbeta^0\|_{\infty} + \|\hbd - \bd^0\|_{\infty} \\
&\leq C\\&\leq \hbeta_{i}+\hd_{i}-\lambda, \quad \forall i \in \{ j \in \tilde{A}: \hbeta_{j}+\hd_{j}>\lambda\}
\end{align*}
where the last inequality uses the definition that  $C = \min_{i\in \tilde{A}} ||\hbeta_i + \hd_i| - \lambda|$.
This implies that $\hbeta_i + \hd_i >\lambda \Longrightarrow \beta_{i}^{0} + d_{i}^{0} >\lambda$ (similarly, we can show  $\hbeta_i + \hd_i <-\lambda \Longrightarrow \beta_{i}^{0} + d_{i}^{0} <-\lambda$), i.e., $A =
\{i: |\hbeta_i + \hd_i| >\lambda \}\subseteq A_{0} = \{i : |\beta^{0}_{i} +  d^{0}_{i}| >\lambda\}$.
Meanwhile, by the same argument we can show that $ |\hbeta_{i} +  \hd_{i}|  <\lambda \Longrightarrow |\beta^{0}_{i} +  d^{0}_{i}| <\lambda $, i.e.,
$ A_{0} \subseteq \bar{A} =
\{i: |\hbeta_i + \hd_i| \geq\lambda \}$.
  Then by the second equation of  \eqref{K2} and the definition of soft threshold operator  we get $\bd_{\bar{A}}=\lambda \sgn(\hbd_{\bar{A}}+\hbbeta_{\bar{A}})$ which implies $\bd_{A_{0}}=\lambda \sgn(\bd_{A_{0}}+\bbeta_{A_{0}})$. This together with
  the first equation of  \eqref{K2} and \eqref{e213} implies
  \begin{equation*}
X_{A_0}^{\prime}X_{A_0} \bbeta_{A_0} + n\bd_{A_0} = X_{A_0}^{\prime}\y =X_{A_0}^{\prime}X_{A_{0}}  \bbeta^{1}_{A_0} + n\bd^1_{A_0}.
\end{equation*}
Then we get $X_{A_0}^{\prime}X_{A_0} (\hbbeta_{A_0}-\bbeta^{1}_{A_0})=\textbf{0}$, therefore, $\hbbeta_{A_0}=\bbeta^{1}_{A_0}$ follows from the above equation and the assumption that $\textrm{rank}(X_{A})=|A|$. Let $B^0=(A_0)^c$, by \eqref{e211} and the fact $A\subset A_{0}$  we deduce that $\bbeta^{1}_{B^0} = \textbf{0}=\hbbeta_{B^0}$. Hence,  $\hbbeta = \bbeta^1$.
This completes the proof of Theorem \ref{th5-6}.
\end{proof}

%

In order to prove Lemma \ref{Lem3}, we need  the following two lemmas. Lemma \ref{Lem4} collects some property  on mutual coherence  and Lemma \ref{Lem5} states
that the effect of the noise $\etaa$ can be controlled with high probability.

\begin{lemma}\label{Lem4}
Let $A$, $B$ be disjoint subsets of $S={1,2,...,p}$, with $\abs{A}=a$, $\abs{B}=b$. Let $\nu$ be the mutual coherence of $X$.
Then we have
\begin{align}
\normi{\bX_{B}^{\prime}\bX_{A}\u}  &\leq n a\nu\normi{\u}, \forall  \u \in \mathbb{R}^{\abs{A}},\label{Lem4-1}\\
\norm{\bX_{A}}=\norm{\bX^{\prime}_{A}}&\leq \sqrt{n(1 + (a-1)\nu)}. \label{Lem4-2}
\end{align}
Furthermore, if  $\nu<1/(a-1)$,  then  $\forall  u \in \mathbb{R}^{\abs{A}}$,
\begin{align}
 \normi{(\bX_{A}^{\prime}\bX_{A})\u} &\geq { n (1-(a-1)\nu)} \normi{\u}, \label{Lem4-3}\\
 \normi{(\bX_{A}^{\prime}\bX_{A})^{-1}\u} &\leq\frac{\normi{\u}}{ n (1-(a-1)\nu)},\label{Lem4-4}\\
 \normi{(\bX_{A}^{\prime}\bX_{A}-nI)\u} &\leq { n (1+(a-1)\nu)} \normi{\u}.\label{Lem4-5}
\end{align}
\end{lemma}

\begin{proof}
Let $G =\bX^{\prime}\bX/n$.  $\forall i\in B,$   $|\sum_{j=1}^{a}G_{i,j}u_{j}|\leq\mu a\normi{u},$ which implies \eqref{Lem4-1}.  For any  $i\in A$, by using Gerschgorin's disk theorem,
$ |\norm{G_{A,A}}-G_{i,i}|\leq \sum_{i\neq j=1}^{a}|G_{i,j}|\leq  (a-1)\mu$, i.e.,   \eqref{Lem4-1} holds. Let $i\in A$ such that $\normi{u} = \abs{u_i}$. \eqref{Lem4-3} follows from that $|\sum_{j=1}^{a}G_{i,j}u_{j}|\geq|u_i| - \sum_{i\neq j=1}^{a}\abs{G_{i,j}}|u_{j}| \geq \normi{u} - \mu (a-1)\normi{u}$.
\eqref{Lem4-4}
  follows directly from \eqref{Lem4-3}.
 And  \eqref{Lem4-5} can be showed similarly as the  \eqref{Lem4-3}.
 This complete the proof of Lemma  \ref{Lem4}.
\end{proof}

\begin{lemma}\label{Lem5}
Suppose (A3) holds. We have
\begin{equation}\label{Lem5-1}
 \bP\xkhB{\normi{X^{\prime} \etaa}/n \leq  \la_u}
\geq 1 - \frac{1}{2 \sqrt{\pi\log(p)}}.
\end{equation}
\end{lemma}
\begin{proof}
This inequality follows from standard probabilities calculations.
\end{proof}

Recall that  $\lambda_u = \sigma\sqrt{2\log(p)/n}$, $\delta_u = 3\lambda_{u}$, 
$\gamma = 8/13$,
$\lambda_{0}=\|X^{\prime} y/n\|_{\infty}$ and
$\lambda_{t} = \lambda_{0} \gamma^{t}$, $t= 0, 1, ...$.


%

\medskip\noindent
\textbf{Proof of Lemma \ref{Lem3}.}
\begin{proof}
We first show that under the assumption of Lemma \ref{Lem3},
  \begin{equation}\label{Lem3-1}
  \la_1>10\delta_u
  \end{equation}
  holds \textrm{with probability at least} $1 - \frac{1}{2 \sqrt{\pi\log(p)}}$.
  In fact,
\begin{align*}
\la_1 = \la_0\gamma &= \frac{8}{13}\normi{X^{\prime}\y/n} =  \frac{8}{13}\normi{X^{\prime}(X\bbeta^{\dag}+\etaa)/n}\\
&\geq \frac{8}{13}(\normi{X_{A^{\dag}}^{\prime}X_{A^{\dag}}\bbeta^{\dag}_{A^{\dag}}/n}-\normi{X^{\prime}\etaa/n})\\
&\geq  \frac{8}{13}((1-(T-1)\nu)\normi{\bbeta^{\dag}}-\la_u) \quad
\textrm{W. H. P.}\\
& > \frac{8}{13} (\frac{3}{4}26\delta_u-\frac{\delta_u}{3})\\
&>10\delta_u
\end{align*}
where the first inequality is the triangle inequality, the second inequality  uses  Lemma \eqref{Lem4-3}-\eqref{Lem5-1}, and the third one follows uses assumption  (A1)-(A2).
Here in the third line, ``W. H. P.'' stands for with high probability, that is, with
probability at least $1 - {1}/{(2 \sqrt{\pi\log(p)})}.$
  Then it follow from \eqref{Lem3-1} and the definition of $\lambda_t$ that there exist an integer $ N \in [1, \log_{\gamma}(\frac{10\delta_u}{\lambda_{0}}))$
 such that \begin{equation}\label{Lem3-2}
 \lambda_{N}  > 10\delta_u \geq \lambda_{N+1}
 \end{equation}
  holds with high probability.
 It follows from  assumption (A2) and \eqref{Lem3-2} that
 $\la_{N+1}=\la_N 8/13\leq\ 10 \delta_u\leq |\bbeta^{\dag}|_{min}10/26$, which implies that with high probability
 $|\bbeta^{\dag}|_{min}> 8\la_N/5$ holds. This complete the proof of Lemma  \ref{Lem3}.
\end{proof}

The  main idea behind the  proof of Theorem \ref{th6}  is that under assumption (A1)-(A3) the active generated by SNAP is contained in the
underlying target support and increase  in some sense with high probability.  To show this we need the
following two Lemmas. Lemma \ref{Lem7} gives one step error estimations  of SNA (Algorithm \ref{alg1}) and Lemma \ref{Lem7}
shows that
some monotone property of  the active set.
\begin{lemma}\label{Lem7}
Suppose assumption  (A1) holds.
Let  $A^{k},B^{k},\bbkk,\dkk$ are generated by $Sna(\bbeta^{0}, \d^{0}, \lambda, \bar{\lambda},K)$ with $\lambda>\bar{\lambda}=\frac{9\lambda}{10}+\delta_u$.  Denote  $E^k =  A^\dag\backslash A^k$ and $i_k = \{i\in B^k:|\bbp_{i}|= \normi{\bbp}\} $.
If   $A^k \subset A^\dag$, then \textrm{with probability at least} $1 - \frac{1}{2 \sqrt{\pi\log(p)}}$ we have
\begin{align}
\normi{\bbkk_{A^k}+\dkk_{A^k}-\bbp_{A^k}}&<\frac{1}{3}|\bbp_{i_k}|+\frac{\la}{30},\label{Lem7-1}\\
|\bbkk_{i}+\dkk_{i}|&>|\bbp_i|-\frac{1}{3}|\bbp_{i_k}|-\frac{\la}{30}, \forall i\in A^k, \label{Lem7-2}\\
|\dkk_i|&<\frac{1}{3}|\bbp_{i_k}|+\frac{\la}{30}, \label{Lem7-3}\\
|\dkk_{i_k}|&>\frac{2}{3}|\bbp_{i_k}|-\frac{\la}{30}.\label{Lem7-4}
\end{align}
\end{lemma}
\begin{proof}
Since $\bbkk,\dkk$ are generated by SNA with $\la>\bar{\lambda}=\frac{9\lambda}{10}+\delta_u$,
  $A^k \subset A^{\dag}$, $E^k =  A^\dag\backslash A^k$ and  $\y = X_{A^{\dag}}\bbp_{{A^{\dag}}}+\etaa$ we have
  \begin{equation}\label{Lem7-5}
  \bbkk_{A^k} = (X_{A^k}^{\prime}X_{A^k})^{-1}(X_{A^k}^{\prime}(X_{A^{k}}\bbp_{{A^{k}}}+X_{E^{k}}\bbp_{{E^{k}}}+\etaa)-n\dkk)
  \end{equation}
  and
  \begin{align*}
 &\normi{\bbkk_{A^k}+\dkk_{A^k}-\bbp_{A^k}}\leq
 \normi{(X_{A^k}^{\prime}X_{A^k})^{-1}(X_{A^k}^{\prime}(X_{E^{k}}\bbp_{{E^{k}}}+\etaa))}\\ &+\normi{(X_{A^k}^{\prime}X_{A^k})^{-1}(X_{A^k}^{\prime}X_{A^k}-nI)\dkk}\\
 &\leq \frac{n|E^k|\nu|\bbp_{i_k}|+\normi{X^{\prime}_{A^k}\etaa}}{n(1-(|A^k|-1)\nu)}+\frac{n(|A^k|-1)\nu}{n(1-(|A^k|-1)\nu)}(\la-\bar{\lambda})\\
 &<\frac{T\nu|\bbp_{i_k}|+\la_u}{(1-T\nu)}+\frac{T\nu}{(1-T\nu)}(\la-(\frac{9\lambda}{10}+\delta_u)) \quad \textrm{W.H.P}. \\
 &\leq\frac{1}{3}|\bbp_{i_k}|+\frac{\la}{30}
 \end{align*}
 where the first inequality uses \eqref{Lem7-5} and the triangle inequality, the second inequality uses  \eqref{Lem4-1},  \eqref{Lem4-4} and \eqref{Lem4-4},    the third inequality uses \eqref{Lem5-1}, the last inequality uses  assumption (A1). Thus,   (\ref{Lem7-1}) holds.  Then, (\ref{Lem7-2})  follows from (\ref{Lem7-1}) and the triangle inequality.
 $\forall i\in B^k,$
\begin{align*}
\abs{\dkk_{i}} &= |X_{i}^{\prime}(X_{\Ak}(\bbp_{\Ak}-\bbkk_{\Ak} - \dkk_{\Ak})+X_{\Ak}\dkk_{\Ak}+X_{E^k}\bbp_{E^k}+\etaa)/n|\\
&\leq |X_{i}^{\prime}X_{\Ak}(\bbp_{\Ak}-\bbkk_{\Ak} - \dkk_{\Ak})|+|X_{i}^{\prime}X_{\Ak}\dkk_{\Ak} +X_{i}^{\prime}X_{E^k}\bbp_{E^k}+X_{i}^{\prime}\etaa|/n\\
&\leq \nu|\Ak|\normi{\bbkk_{A^k}+\dkk_{A^k}-\bbp_{A^k}}+\nu|\Ak|(\la- \bar{\lambda})+\nu|E^k||\bbp_{i_k}|+\la_u \quad \textrm{W.H.P.} \\
&< \frac{1}{4}( \frac{1}{3}|\bbp_{i_k}|+\frac{\la}{30})+\frac{1}{4}(\la- \bar{\lambda})+\frac{1}{4}|\bbp_{i_k}|+\la_u\\
&= \frac{1}{3}|\bbp_{i_k}|+\frac{\la}{30}
\end{align*}
 where the first equality uses \eqref{Lem7-5},  the first   inequality is the triangle inequality, the second inequality is   due to   \eqref{Lem4-1} and \eqref{Lem5-1}, and the third inequality uses   \eqref{Lem7-1}, i.e., \eqref{Lem7-3} holds.
 Observing  $i_k\in E^k$ and (\ref{Lem7-5}) we get
 \begin{align*}
\abs{\dkk_{i_k}} &= |X_{i_k}^{\prime}(X_{\Ak}(\bbp_{\Ak}-\bbkk_{\Ak} - \dkk_{\Ak})+X_{\Ak}\dkk_{\Ak}+ X_{i_k}\bbp_{i_k}+ X_{E^{k} \backslash i_{k}}\bbp_{E^{k}\backslash i_{k}}+\etaa)/n|\\
&\geq |\bbp_{i_k}|-|X_{i}^{\prime}X_{\Ak}(\bbp_{\Ak}-\bbkk_{\Ak} - \dkk_{\Ak})|-|X_{i}^{\prime}(X_{\Ak}\dkk_{\Ak} +X_{E^{k} \backslash i_{k}}\bbp_{E^{k}\backslash i_{k}}+\etaa|/n\\
&\geq |\bbp_{i_k}| -\nu|\Ak|\normi{\bbkk_{A^k}-\dkk_{A^k}-\bbp_{A^k}}-\nu|\Ak|(\la- \bar{\lambda})-\nu|E^k||\bbp_{i_k}|-\la_u \quad \textrm{W.H.P}, \\
&> |\bbp_{i_k}|-\frac{1}{4}( \frac{1}{3}|\bbp_{i_k}|+\frac{\la}{30})-\frac{1}{4}(\la- \bar{\lambda})-\frac{1}{4}|\bbp_{i_k}|-\la_u \\ 
&> \frac{2}{3}|\bbp_{i_k}|-\frac{\la}{30}
\end{align*}
 where the first inequality is the triangle inequality,  the second  inequality is due to  Lemma \eqref{Lem4-1} and \eqref{Lem5-1}, and the third one uses  \ref{Lem7-1}, i.e., (\ref{Lem7-4}) holds. This complete the proof of Lemma  \ref{Lem7}.
\end{proof}

For a given $\tau>0$, we define
$S_{\la,\tau} = \{i:|\bbp_i|\geq \lambda\tau\}$.
\begin{lemma}\label{Lem8}
 Suppose  assumption (A1) hold. Let $\kappa = \frac{8}{5}$  and $\tau = \kappa$  or $\kappa+1$.
 Denote  $E^k =  A^\dag\backslash A^k$ and $i_k = \{i\in B^k:|\bbp_{i}|= \normi{\bbp}\} $.
 If $S_{\la,\tau}  \subset \Ak\subset A^{\dag}$ then $S_{\la,\tau}   \subset \Akk\subset A^{\dag}.$  Meanwhile, if $S_{\la,\kappa+1} \subset \Ak\subset A^{\dag}$ and $S_{\la,\kappa} \nsubseteq \Ak$ then $|\bbp_{i_k}|>|\bbp_{i_{k+1}}|.$
\end{lemma}

\begin{proof}
Assume $S_{\la,\tau} \subset \Ak\subset A^{\dag}$. Since  $E^k=A^{\dag} \backslash \Ak$ and $i_k \in E^k$, we get
$i_k\notin A^k$ which implies $|\bbp_{i_k}|<\la\tau.$   $\forall i\in S_{\la,\tau} \subset \Ak$. By using (\ref{Lem7-2})  we have
$$|\bbkk_{i}+\dkk_{i}|>|\bbp_i|-\frac{1}{3}|\bbp_{i_k}|-\frac{\la}{30}
>\la\tau -\frac{1}{3}\la\tau-\frac{\la}{30}>\la, $$ which implies $i\in \Akk$, i.e., $S_{\la,\kappa} \subset \Akk$ holds.     $\forall i \in (\Ap)^c \subset B^k$. By using (\ref{Lem7-3})  we get
\begin{equation}\label{Lem8-1}
|\bbkk_i+ \dkk_i|=|\dkk_i|<\frac{1}{3}|\bbp_{i_k}|+\frac{\la}{30}<\left\{
    \begin{array}{ll}
   \lambda,    \quad &\text{$\tau=\kappa+1,$}\\
   \frac{\kappa}{\kappa+1}\la  ,  \quad &\text{$\tau=\kappa,$}
    \end{array}
  \right.
  \end{equation}
i.e., $i\notin \Akk$ which implies  $\Akk\subset A^{\dag}$.
Next we turn to the second assertion.
Assume   $S_{\la,\kappa+1} \subset \Ak\subset A^{\dag}$, $S_{\la,\kappa} \nsubseteq \Ak.$
It suffice to show  all the  elements of $|\bbp|$ that larger than $|\bbp_{i_k}|$ move into  $\Akk$.
It follows from the definition  of $S_{\la,\kappa}$, $S_{\la,\kappa+1}$ and    $i_k \in E^k=A^{\dag} \backslash \Ak$ that  $i_k \in S_{\la,\kappa}\backslash S_{\la,\kappa+1}$, i.e., $|\bbp_{i_k}| \in [\la\kappa,\la(\kappa+1)).$  By using (\ref{Lem7-4})  we have
$$|\bbkk_{i_k}+\dkk_{i_k}| = |\dkk_{i_k}| >\frac{2}{3}|\bbp_{i_k}|-\frac{1}{30}\la>\frac{2}{3}\la\kappa-\frac{1}{30}\la>\la,$$ which implies $i_k\in \Akk.$ Let $i\in A^{k}$ satisfy $|\bbp_i|\geq|\bbp_{i_k}|$. Then it follows from
(\ref{Lem7-2})  that
\begin{align*}
|\bbkk_{i}+\dkk_{i}|&>|\bbp_i|-\frac{1}{3}|\bbp_{i_k}|-\frac{\la}{30}\\
&>\frac{2}{3}|\bbp_{i_k}|-\frac{1}{30}\la\\
&>\frac{2}{3}\la\kappa-\frac{1}{30}\la>\la,
\end{align*}
which implies $i \in \Akk.$  This complete the proof of Lemma  \ref{Lem8}.
\end{proof}

With the above preparation, we now give the prove of Theorem  \ref{th6}.

\medskip\noindent
\textbf{Proof of Theorem \ref{th6}.}
\begin{proof}
Let $\overline{\la}_t = \frac{9}{10}\la_t+{\delta_u}$.
By  using Lemma \ref{Lem3} and  the definition of $\la_t$ and we get $\la_t > \overline{\la}_t, t = 0, 1,...,N.$
At the $t_{th}$  knot  of $Snap(\lambda_{0},\gamma,N,K)$, suppose it takes Algorithm  $Sna(\bbeta^0,\d^0, \lambda_t, \overline{\la}_t,K)$  $k_t$ iterations to get the solution $(\hbbeta(\la_{t}), \hbd(\la_{t}))$, where $(\bbeta^0,\d^0)=(\hbbeta(\la_{t-1}), \hbd(\la_{t-1}))$ and   $k_t\leq K$ by the definition of SNAP.
We denote the approximate primal dual solution pair  and  active set generated in $Sna(\hbbeta(\la_{t-1}), \hbd(\la_{t-1}), \lambda_t, \overline{\lambda_t},K)$
by $(\bbeta^{k}_{t},\d^{k}_{t})$ and  $A^k_{t}$, respectively, $k=0,1,...,k_t.$
By the construction of SNAP we have $(\bbeta^{k_t}_{t},\d^{k_t}_{t})=(\hbbeta(\la_{t}), \hbd(\la_{t}))$, i.e, the solution at the $t_{th}$ stage is the initial value for the $t+1$ stage which implies
\begin{equation}\label{th6cor}
A^{k_t}_{t}\subseteq A^{0}_{t+1}.
\end{equation}
We claim  that
\begin{equation}\label{claim}
S_{\la_t,\kappa+1}\subseteq A^0_{t}\subseteq\Ap,  t=0,1,...,N.
\end{equation}
\begin{equation}\label{claim2}
S_{\la_t,\kappa}\subseteq A^{k_t}_{t}\subseteq\Ap, t=0,1,...,N.
\end{equation}
We prove the above two claims by mathematical induction.
First we show that $\emptyset=S_{\la_0,\kappa+1}\subseteq A^0_{0}\subseteq\Ap$.
Let $|\bbp_i| = \normi{\bbp}$.
\begin{align}
(\kappa+1)\la_0 &=\frac{13}{5}\normi{X^{\prime}\y/n} = \frac{13}{5}\normi{X^{\prime}(X\bbp+\etaa)/n}\notag\\
&\geq\frac{13}{5}(\normi{X_{A^{\dag}}^{\prime}X_{A^{\dag}}\bbp_{A^{\dag}}/n}
-\normi{X^{\prime}\etaa/n})\notag\\
&\geq \frac{13}{5}((1-(T-1)\nu)|\bbp_i|-\la_u),\quad  W.H.P\notag\\
& > \frac{13}{5}(\frac{3}{4}|\bbp_i|-\la_u)\notag\\
& > |\bbp_i|
\end{align}
where the first inequality is the triangle equation and the second inequality uses  \eqref{Lem4-3} and \eqref{Lem5-1}, the third inequality uses assumption (A1),
and the  last inequality is derive from  assumption (A2). 
This implies  $\emptyset=S_{\la_0,\kappa+1}.$
By the construction of $Snap(\lambda_{0},\gamma,N,K)$ we get  $A^0_0 = \{j: |X_{j}^{\prime}\y/n|>\la_0 = \normi{X^{\prime}\y/n}\}=\emptyset$.
Therefore,   (\ref{claim}) holds   when $t=0$.
Now we suppose (\ref{claim}) holds for some  $t\geq0.$
Then by the first assertion of Lemma \ref{Lem8} we get
\begin{equation}\label{th6.5}
S_{\la_t,\kappa+1}\subseteq A^{k}_{t}\subseteq\Ap,  k = 0,1,...,k_t.
\end{equation}
By the stopping rule of $Sna(\beta^0,d^0, \lambda_t, \overline{\lambda_t},K)$  it holds
 either $A^{k_t}_{t} =A^{k_t-1}_{t} $ or  $k_t = K \geq T$ when   it  stops.
 In both cases, by using (\ref{th6.5})  and the second assertion of Lemma \ref{Lem8}  we get
\begin{equation*}
S_{\la_t,\kappa}\subseteq A^{k_t}_{t}\subseteq\Ap,
\end{equation*}
i.e., (\ref{claim2}) holds for this given $t$.
Observing the relation  $S_{\la_{t+1},\kappa+1} = S_{\la_{t},\kappa}$  and  (\ref{Lem8-1})-(\ref{th6cor})
we get $S_{\la_{t+1},\kappa+1}\subseteq A^{0}_{t+1}\subseteq \Ap,$ i.e., (\ref{claim}) holds for $t+1.$
Therefore, (\ref{claim}) - (\ref{claim2}) are verified by  mathematical induction on $t$.
That is all the active set generated in SNAP is contained in $\Ap$.
Therefore, by Lemma \ref{Lem3} we get $$\Ap\subseteq S_{\la_N,\kappa} \subseteq A^{k_N}_{N} \subseteq \Ap,$$ i.e.,
\begin{equation}\label{th6.6}
\textrm{supp}(\hbbeta(\lambda_N)) =   A^{\dag}.
\end{equation}
Then,
\begin{align*}
\normi{\bbp-\hbbeta(\la_N)} &= \normi{\bbp_{A^{\dag}}-(X^{\prime}_{\Ap}X_{\Ap})^{-1}(\tby_{A^{\dag}} - n\hbd(\la_N)_{A^{\dag}})}\\
&=\normi{\bbp_{A^{\dag}}-(X^{\prime}_{\Ap}X_{\Ap})^{-1}(X^{\prime}_{\Ap}(X_{\Ap}\bbp_{\Ap}+\etaa) - n\hbd(\la_N)_{A^{\dag}})}\\
&\leq \frac{\normi{X^{\prime}_{\Ap}\etaa}+n(\la_N-\overline{\lambda_N})}{n(1-T\nu)}\\
&<\frac{\la_u+\frac{\la_N}{10}-3\la_u}{1-\frac{1}{4}} \quad \textrm{W.H.P.} \\
\leq\frac{\frac{39}{8}\la_{u}-2\la_u}{\frac{3}{4}}=\frac{23}{6}\la_u,
\end{align*}
where the first  inequality uses \eqref{Lem4-4}, the second inequality uses \eqref{Lem5-1}, and last inequality uses
Lemma \ref{Lem3},
i.e., (\ref {Th62}) holds.
The sign consistency (\ref {Th61}) follows directly from (\ref{th6.6}), (\ref {Th62})  and assumption (A2).
This complete the proof of Theorem \ref{th6}.
\end{proof}

\subsection{Details in \algref{alg3}}
We now describe in detail the quantities in the $k_{th}$ iteration in
Algorithm \ref{alg3}. This paves the way for showing that Algorithm \ref{alg1} is a
specialization of Algorithm \ref{alg3}.
At $\z^{k}=(\bbeta^{k\prime}, \d^{k\prime})'$,
we define  $A_{k}$ and $B_{k}$ by (\ref{eac}).
By a similar reordering of $(\bbeta^{k\prime}, \d^{k\prime})'$,
$F_{1}(\z^{k})$ and $F_{2}(\z^{k})$ as   concerning the Newton derivative of $F$ in Theorem \ref{th4},  and using  the  definition  of $T_{\lambda}(\cdot)$,  we get
\begin{equation}\label{FF}
\z^{k}=\left(
         \begin{array}{c}
         \d_{A_{k}}^{k} \\
          \bbeta_{B_{k}}^{k} \\
           \bbeta_{A_{k}}^{k} \\
          \d_{B_{k}}^{k} \\
         \end{array}
       \right),
F(\z^{k})=
\left[\begin{array}{c}
 - \d_{A_{k}}^{k}  + \lambda \sgn(\bbeta_{A_{k}}^{k} + \d_{A_{k}}^{k}) \\
\bbeta_{B_{k}}^{k}  \\
  G_{A_{k}A_{k}} \bbeta_{A_{k}}^{k} + G_{A_{k}B_{k}} \bbeta_{B_{k}}^{k}+ n\d_{A_{k}}^{k} -{\tby}_{A_{k}}^{k}      \\
G_{B_{k}A_{k}} \bbeta_{A_{k}}^{k} + G_{B_{k}B_{k}} \bbeta_{B_{k}}^{k}+ n\d_{B_{k}}^{k} -{\tby}_{B_{k}}^{k}
\end{array}\right].
\end{equation}
Then, by using Theorem \ref{th4} and noting that $G_{B_{k}A_{k}} = X_{B_{k}}^{\prime}X_{A_{k}}, G_{A_{k}B_{k}} = X_{A_{k}}^{\prime}X_{B_{k}}$ we have  $H_{k} \in \nabla_{N}F(\z^{k})$, where
\begin{equation}\label{ndk}
\hspace{-0.1cm}H_{k}\hspace{-0.1cm}  =  \hspace{-0.1cm}
\begin{bmatrix}
& -I_{A_{k}A_{k}}    & {\0}    &{\0}  & \hspace{-0.1cm} {\0} \\ \\
& {\0}          & I_{B_{k}B_{k}}   &{\0} &\hspace{-0.1cm}{\0} \\ \\
& nI_{A_{k}A_{k}}    & X_{A_{k}}^{\prime}X_{B_{k}}                                   &G_{A_{k}A_{k}}      & \hspace{-0.1cm}{\0} \\ \\
& {\0}          & G_{B_{k}B_{k}} &X_{B_{k}}^{\prime}X_{A_{k}}  & nI_{B_{k}B_{k}}
\end{bmatrix}.
\end{equation}

Algorithm \ref{alg3} is well defined if we
choose $H_{k}$ in the form of (\ref{ndk}), since $H_{k}$ is invertible
as shown in  Theorem \ref{th4}.


In Section 2.1 we derived Algorithm \ref{alg1} in  an intuitive way.
We now verify that Algorithm \ref{alg1} is indeed Algorithm \ref{alg3} in a form that can be easily and efficiently implemented computationally. Let
$$D^{k}=\left(
         \begin{array}{c}
         D^{\d}_{A_{k}} \\
          D^{\bbeta}_{B_{k}} \\
           D^{\bbeta}_{A_{k}} \\
          D^{\d}_{B_{k}} \\
         \end{array}
          \right)$$
and substitute  (\ref{FF}) and (\ref{ndk}) into (\ref{ssnd}) we get
\begin{align}
 \d_{A_{k}}^{k} + D^{\d}_{A_{k}}    &=  \lambda \sgn(\bbeta_{A_{k}}^{k} + \d_{A_{k}}^{k}), \label{eqv1}\\
 \bbeta^{k}_{B_{k}} + D^{\bbeta}_{B_{k}}  &= {\0},\label{eqv2}\\
G_{A_{k}A_{k}} ( \bbeta^{k}_{A_{k}}+D^{\bbeta}_{A_{k}})&= {\tby}_{A_{k}} - n( \d_{A_{k}}^{k}+D^{\d}_{A_{k}})- X_{A_{k}}^{\prime}X_{B_{k}}( \bbeta^{k}_{B_{k}}+D^{\bbeta}_{B_{k}}),  \label{eqv3}\\
n(\d_{B_{k}}^{k} + D^{\d}_{B_{k}}) &= {\tby}_{B_{k}} - X_{B_{k}}^{\prime}
X_{A_{k}}(\bbeta^{k}_{A_{k}}+D^{\bbeta}_{A_{k}}) - G_{B_{k}B_{k}}( \bbeta^{k}_{B_{k}} + D^{\bbeta}_{B_{k}}) . \label{eqv4}
\end{align}
Observing the relationship
(by (\ref{ssnup})),
$$\left( \begin{array}{c}
         \d_{A_{k}}^{k+1} \\
          \bbeta_{B_{k}}^{k+1} \\
           \bbeta_{A_{k}}^{k+1} \\
          \d_{B_{k}}^{k+1} \\
         \end{array}
       \right)
       = \left(
         \begin{array}{c}
         \d_{A_{k}}^{k}+D^{\d}_{A_{k}} \\
          \bbeta_{B_{k}}^{k}+D^{\bbeta}_{B_{k}} \\
           \bbeta_{A_{k}}^{k} +D^{\bbeta}_{A_{k}}\\
          \d_{B_{k}}^{k}+ D^{d}_{B_{k}} \\
         \end{array}
       \right).
       $$
and substituting (\ref{eqv1}) - (\ref{eqv2}) into  (\ref{eqv3})-(\ref{eqv4}), we obtain (\ref{e211}) - (\ref{e214}), which are the computational steps
in Algorithm \ref{alg1}.

\bibliographystyle{apalike}
\nocite{*}
\bibliography{ref_snap_v3}

\begin{thebibliography}{}

\bibitem[Agarwal et~al., 2012]{agarwal2012fast}
Agarwal, A., Negahban, S., and Wainwright, M.~J. (2012).
\newblock Fast global convergence of gradient methods for high-dimensional
  statistical recovery.
\newblock {\em The Annals of Statistics}, 40(5):2452--2482.

\bibitem[Beck and Teboulle, 2009]{beck2009fast}
Beck, A. and Teboulle, M. (2009).
\newblock A fast iterative shrinkage-thresholding algorithm for linear inverse
  problems.
\newblock {\em SIAM Journal on Imaging Sciences}, 2(1):183--202.

\bibitem[Becker et~al., 2011]{becker2011nesta}
Becker, S., Bobin, J., and Cand{\`e}s, E.~J. (2011).
\newblock Nesta: A fast and accurate first-order method for sparse recovery.
\newblock {\em SIAM Journal on Imaging Sciences}, 4(1):1--39.

\bibitem[Boyd et~al., 2011]{boyd2011distributed}
Boyd, S., Parikh, N., Chu, E., Peleato, B., and Eckstein, J. (2011).
\newblock Distributed optimization and statistical learning via the alternating
  direction method of multipliers.
\newblock {\em Foundations and Trends{\textregistered} in Machine learning},
  3(1):1--122.

\bibitem[Breheny, 2018]{breheny2018marginal}
Breheny, P. (2018).
\newblock Marginal false discovery rates for penalized regression models.
\newblock {\em Biostatistics}, https://doi.org/10.1093/biostatistics/kxy004.

\bibitem[Breheny and Huang, 2011]{breheny2011coordinate}
Breheny, P. and Huang, J. (2011).
\newblock Coordinate descent algorithms for nonconvex penalized regression,
  with applications to biological feature selection.
\newblock {\em The Annals of Applied Statistics}, 5(1):232--253.

\bibitem[Cand\`es and Plan, 2009]{candes2009near}
Cand\`es, E.~J. and Plan, Y. (2009).
\newblock Near-ideal model selection by $\ell_1$ minimization.
\newblock {\em The Annals of Statistics}, 37(5A):2145--2177.

\bibitem[Cand{\`e}s et~al., 2006]{candes2006robust}
Cand{\`e}s, E.~J., Romberg, J., and Tao, T. (2006).
\newblock Robust uncertainty principles: Exact signal reconstruction from
  highly incomplete frequency information.
\newblock {\em IEEE Transactions on Information Theory}, 52(2):489--509.

\bibitem[Cand{\`e}s and Tao, 2006]{candes2006near}
Cand{\`e}s, E.~J. and Tao, T. (2006).
\newblock Near-optimal signal recovery from random projections: Universal
  encoding strategies?
\newblock {\em IEEE Transactions on Information Theory}, 52(12):5406--5425.

\bibitem[Carpentier, 2015]{carpentier2015implementable}
Carpentier, A. (2015).
\newblock Implementable confidence sets in high dimensional regression.
\newblock In {\em Proceedings of the Eighteenth International Conference on
  Artificial Intelligence and Statistics}, volume~38 of {\em Proceedings of
  Machine Learning Research}, pages 120--128, San Diego, California, USA. PMLR.

\bibitem[Chen and Chen, 2008]{chen2008extended}
Chen, J. and Chen, Z. (2008).
\newblock Extended bayesian information criteria for model selection with large
  model spaces.
\newblock {\em Biometrika}, 95(3):759--771.

\bibitem[Chen and Chen, 2012]{chen2012extended}
Chen, J. and Chen, Z. (2012).
\newblock {Extended BIC for small-$n$-large-$P$ sparse GLM}.
\newblock {\em Statistica Sinica}, 22:555--574.

\bibitem[Chen et~al., 2017]{chen2017efficient}
Chen, L., Sun, D., and Toh, K.-C. (2017).
\newblock {An efficient inexact symmetric Gauss--Seidel based majorized ADMM
  for high-dimensional convex composite conic programming}.
\newblock {\em Mathematical Programming}, 161(1-2):237--270.

\bibitem[Chen et~al., 1998a]{chen1998atomic}
Chen, S.~S., Donoho, D.~L., and Saunders, M.~A. (1998a).
\newblock Atomic decomposition by basis pursuit.
\newblock {\em SIAM Journal on Scientific Computing}, 20(1):33--61.

\bibitem[Chen et~al., 2000]{chen2000smoothing}
Chen, X., Nashed, Z., and Qi, L. (2000).
\newblock Smoothing methods and semismooth methods for nondifferentiable
  operator equations.
\newblock {\em SIAM Journal on Numerical Analysis}, 38(4):1200--1216.

\bibitem[Chen et~al., 1998b]{chen1998global}
Chen, X., Qi, L., and Sun, D. (1998b).
\newblock Global and superlinear convergence of the smoothing {Newton} method
  and its application to general box constrained variational inequalities.
\newblock {\em Mathematics of Computation of the American Mathematical
  Society}, 67(222):519--540.

\bibitem[Combettes and Wajs, 2005]{combettes2005signal}
Combettes, P.~L. and Wajs, V.~R. (2005).
\newblock Signal recovery by proximal forward-backward splitting.
\newblock {\em Multiscale Modeling and Simulation}, 4(4):1168--1200.

\bibitem[Daubechies et~al., 2004]{daubechies2004iterative}
Daubechies, I., Defrise, M., and De~Mol, C. (2004).
\newblock An iterative thresholding algorithm for linear inverse problems with
  a sparsity constraint.
\newblock {\em Communications on Pure and Applied Mathematics},
  57(11):1413--1457.

\bibitem[Donoho et~al., 2006]{donoho2006stable}
Donoho, D.~L., Elad, M., and Temlyakov, V.~N. (2006).
\newblock Stable recovery of sparse overcomplete representations in the
  presence of noise.
\newblock {\em IEEE Transactions on Information Theory}, 52(1):6--18.

\bibitem[Donoho and Huo, 2001]{donoho2001uncertainty}
Donoho, D.~L. and Huo, X. (2001).
\newblock Uncertainty principles and ideal atomic decomposition.
\newblock {\em IEEE Transactions on Information Theory}, 47(7):2845--2862.

\bibitem[Donoho and Johnstone, 1995]{donoho1995adapting}
Donoho, D.~L. and Johnstone, I.~M. (1995).
\newblock Adapting to unknown smoothness via wavelet shrinkage.
\newblock {\em Journal of the American Statistical Association},
  90(432):1200--1224.

\bibitem[Donoho and Tsaig, 2008]{donoho2008fast}
Donoho, D.~L. and Tsaig, Y. (2008).
\newblock Fast solution of $\ell_1$-norm minimization problems when the
  solution may be sparse.
\newblock {\em IEEE Transactions on Information Theory}, 54(11):4789--4812.

\bibitem[Efron et~al., 2004]{efron2004least}
Efron, B., Hastie, T., Johnstone, I., and Tibshirani, R. (2004).
\newblock Least angle regression.
\newblock {\em The Annals of Statistics}, 32(2):407--499.

\bibitem[Fan and Li, 2001]{fan2001variable}
Fan, J. and Li, R. (2001).
\newblock Variable selection via nonconcave penalized likelihood and its oracle
  properties.
\newblock {\em Journal of the American Statistical Association},
  96(456):1348--1360.

\bibitem[Fan and Lv, 2008]{fan2008sure}
Fan, J. and Lv, J. (2008).
\newblock Sure independence screening for ultrahigh dimensional feature space.
\newblock {\em Journal of the Royal Statistical Society: Series B (Statistical
  Methodology)}, 70(5):849--911.

\bibitem[Friedman et~al., 2007]{friedman2007pathwise}
Friedman, J., Hastie, T., H{\"o}fling, H., and Tibshirani, R. (2007).
\newblock Pathwise coordinate optimization.
\newblock {\em The Annals of Applied Statistics}, 1(2):302--332.

\bibitem[Friedman et~al., 2010]{friedman2010regularization}
Friedman, J., Hastie, T., and Tibshirani, R. (2010).
\newblock Regularization paths for generalized linear models via coordinate
  descent.
\newblock {\em Journal of Statistical Software}, 33(1):1--22.

\bibitem[Fu, 1998]{fu1998penalized}
Fu, W.~J. (1998).
\newblock Penalized regressions: the bridge versus the lasso.
\newblock {\em Journal of Computational and Graphical Statistics},
  7(3):397--416.

\bibitem[Golub and Van~Loan, 1996]{golub1996matrix}
Golub, G.~H. and Van~Loan, C.~F. (1996).
\newblock {\em Matrix computations}.
\newblock Johns Hopkins University Press, Baltimore, 3 edition.

\bibitem[Han et~al., 2017]{han2017linear}
Han, D., Sun, D., and Zhang, L. (2017).
\newblock Linear rate convergence of the alternating direction method of
  multipliers for convex composite programming.
\newblock {\em Mathematics of Operations Research}, 43(2):622--637.

\bibitem[Huang et~al., 2010]{huang2010variable}
Huang, J., Horowitz, J.~L., and Wei, F. (2010).
\newblock Variable selection in nonparametric additive models.
\newblock {\em The Annals of Statistics}, 38(4):2282--2313.

\bibitem[Huang et~al., 2008]{huang2008adaptive}
Huang, J., Ma, S., and Zhang, C.-H. (2008).
\newblock Adaptive lasso for sparse high-dimensional regression models.
\newblock {\em Statistica Sinica}, 18:1603--1618.

\bibitem[Ito and Kunisch, 2008]{ito2008lagrange}
Ito, K. and Kunisch, K. (2008).
\newblock {\em Lagrange multiplier approach to variational problems and
  applications}.
\newblock SIAM, Philadelphia.

\bibitem[Jiao et~al., 2017]{jiao2017iterative}
Jiao, Y., Jin, B., and Lu, X. (2017).
\newblock Iterative soft/hard thresholding with homotopy continuation for
  sparse recovery.
\newblock {\em IEEE Signal Processing Letters}, 24(6):784--788.

\bibitem[Kim et~al., 2012]{kim2012consistent}
Kim, Y., Kwon, S., and Choi, H. (2012).
\newblock Consistent model selection criteria on high dimensions.
\newblock {\em Journal of Machine Learning Research}, 13:1037--1057.

\bibitem[Kummer, 1988]{kummer1988newton}
Kummer, B. (1988).
\newblock Newton's method for non-differentiable functions.
\newblock {\em Advances in Mathematical Optimization}, 45:114--125.

\bibitem[Li and Osher, 2009]{li2009coordinate}
Li, Y. and Osher, S. (2009).
\newblock Coordinate descent optimization for $\ell^1$ minimization with
  application to compressed sensing; a greedy algorithm.
\newblock {\em Inverse Problems and Imaging}, 3(3):487--503.

\bibitem[Lounici, 2008]{lounici2008sup}
Lounici, K. (2008).
\newblock Sup-norm convergence rate and sign concentration property of {Lasso}
  and {Dantzig} estimators.
\newblock {\em Electronic Journal of Statistics}, 2:90--102.

\bibitem[Lv et~al., 2018]{lv2018oracle}
Lv, S., Lin, H., Lian, H., and Huang, J. (2018).
\newblock Oracle inequalities for sparse additive quantile regression in
  reproducing kernel hilbert space.
\newblock {\em The Annals of Statistics}, 46(2):781--813.

\bibitem[Mazumder et~al., 2011]{mazumder2011sparsenet}
Mazumder, R., Friedman, J.~H., and Hastie, T. (2011).
\newblock {SparseNet}: Coordinate descent with nonconvex penalties.
\newblock {\em Journal of the American Statistical Association},
  106(495):1125--1138.

\bibitem[Meinshausen and B\"uhlmann, 2006]{meinshausen2006high}
Meinshausen, N. and B\"uhlmann, P. (2006).
\newblock High-dimensional graphs and variable selection with the lasso.
\newblock {\em The Annals of Statistics}, 34(3):1436--1462.

\bibitem[Nesterov, 2005]{nesterov2005smooth}
Nesterov, Y. (2005).
\newblock Smooth minimization of non-smooth functions.
\newblock {\em Mathematical Programming}, 103(1):127--152.

\bibitem[Nesterov, 2013]{nesterov2013gradient}
Nesterov, Y. (2013).
\newblock Gradient methods for minimizing composite functions.
\newblock {\em Mathematical Programming}, 140(1):125--161.

\bibitem[Osborne et~al., 2000]{osborne2000new}
Osborne, M.~R., Presnell, B., and Turlach, B.~A. (2000).
\newblock A new approach to variable selection in least squares problems.
\newblock {\em IMA Journal of Numerical Analysis}, 20(3):389--403.

\bibitem[Parikh and Boyd, 2014]{parikh2014proximal}
Parikh, N. and Boyd, S. (2014).
\newblock Proximal algorithms.
\newblock {\em Foundations and Trends{\textregistered} in Optimization},
  1(3):127--239.

\bibitem[Qi and Sun, 1999]{qi1999survey}
Qi, L. and Sun, D. (1999).
\newblock A survey of some nonsmooth equations and smoothing {Newton} methods.
\newblock In {\em Progress in optimization}, pages 121--146. Springer.

\bibitem[Qi et~al., 2000]{qi2000new}
Qi, L., Sun, D., and Zhou, G. (2000).
\newblock A new look at smoothing {Newton} methods for nonlinear
  complementarity problems and box constrained variational inequalities.
\newblock {\em Mathematical Programming}, 87(1):1--35.

\bibitem[Qi and Sun, 1993]{qi1993nonsmooth}
Qi, L. and Sun, J. (1993).
\newblock A nonsmooth version of newton's method.
\newblock {\em Mathematical programming}, 58(1-3):353--367.

\bibitem[Rockafellar, 1970]{rockafellar1970convex}
Rockafellar, R.~T. (1970).
\newblock {\em Convex analysis}.
\newblock Princeton University Press, Princeton.

\bibitem[Saha and Tewari, 2013]{saha2013nonasymptotic}
Saha, A. and Tewari, A. (2013).
\newblock On the nonasymptotic convergence of cyclic coordinate descent
  methods.
\newblock {\em SIAM Journal on Optimization}, 23(1):576--601.

\bibitem[She, 2009]{she2009thresholding}
She, Y. (2009).
\newblock Thresholding-based iterative selection procedures for model selection
  and shrinkage.
\newblock {\em Electronic Journal of Statistics}, 3:384--415.

\bibitem[Shi et~al., 2018a]{shi2018semi}
Shi, Y., Huang, J., Jiao, Y., and Yang, Q. (2018a).
\newblock {Semi-smooth Newton algorithm for non-convex penalized linear
  regression}.
\newblock {\em arXiv preprint arXiv:1802.08895v2}.

\bibitem[Shi et~al., 2018b]{shi2018admm}
Shi, Y., Wu, Y., Xu, D., and Jiao, Y. (2018b).
\newblock {An ADMM with continuation algorithm for non-convex SICA-penalized
  regression in high dimensions}.
\newblock {\em Journal of Statistical Computation and Simulation},
  88(9):1826--1846.

\bibitem[Tan and Huang, 2016]{tan2016bayesian}
Tan, A. and Huang, J. (2016).
\newblock Bayesian inference for high-dimensional linear regression under mnet
  priors.
\newblock {\em Canadian Journal of Statistics}, 44(2):180--197.

\bibitem[Tibshirani, 1996]{tibshirani1996regression}
Tibshirani, R. (1996).
\newblock Regression shrinkage and selection via the lasso.
\newblock {\em Journal of the Royal Statistical Society. Series B
  (Methodological)}, pages 267--288.

\bibitem[Tibshirani et~al., 2012]{tibshirani2012strong}
Tibshirani, R., Bien, J., Friedman, J., Hastie, T., Simon, N., Taylor, J., and
  Tibshirani, R.~J. (2012).
\newblock Strong rules for discarding predictors in lasso-type problems.
\newblock {\em Journal of the Royal Statistical Society: Series B (Statistical
  Methodology)}, 74(2):245--266.

\bibitem[Tropp and Wright, 2010]{tropp2010computational}
Tropp, J.~A. and Wright, S.~J. (2010).
\newblock Computational methods for sparse solution of linear inverse problems.
\newblock {\em Proceedings of the IEEE}, 98(6):948--958.

\bibitem[Tseng, 2001]{tseng2001convergence}
Tseng, P. (2001).
\newblock Convergence of a block coordinate descent method for
  nondifferentiable minimization.
\newblock {\em Journal of Optimization Theory and Applications},
  109(3):475--494.

\bibitem[Tseng and Yun, 2009]{tseng2009coordinate}
Tseng, P. and Yun, S. (2009).
\newblock A coordinate gradient descent method for nonsmooth separable
  minimization.
\newblock {\em Mathematical Programming}, 117:387--423.

\bibitem[Wainwright, 2009]{wainwright2009sharp}
Wainwright, M.~J. (2009).
\newblock Sharp thresholds for high-dimensional and noisy sparsity recovery
  using $\ell_1$-constrained quadratic programming (lasso).
\newblock {\em IEEE Transactions on Information Theory}, 55(5):2183--2202.

\bibitem[Wang et~al., 2009]{wang2009shrinkage}
Wang, H., Li, B., and Leng, C. (2009).
\newblock Shrinkage tuning parameter selection with a diverging number of
  parameters.
\newblock {\em Journal of the Royal Statistical Society: Series B (Statistical
  Methodology)}, 71(3):671--683.

\bibitem[Wang et~al., 2007]{wang2007tuning}
Wang, H., Li, R., and Tsai, C.-L. (2007).
\newblock Tuning parameter selectors for the smoothly clipped absolute
  deviation method.
\newblock {\em Biometrika}, 94(3):553--568.

\bibitem[Wang et~al., 2013]{wang2013calibrating}
Wang, L., Kim, Y., and Li, R. (2013).
\newblock Calibrating nonconvex penalized regression in ultra-high dimension.
\newblock {\em The Annals of Statistics}, 41(5):2505--2536.

\bibitem[Wu and Lange, 2008]{wu2008coordinate}
Wu, T.~T. and Lange, K. (2008).
\newblock Coordinate descent algorithms for lasso penalized regression.
\newblock {\em The Annals of Applied Statistics}, 2(1):224--244.

\bibitem[Xiao and Zhang, 2013]{xiao2013proximal}
Xiao, L. and Zhang, T. (2013).
\newblock A proximal-gradient homotopy method for the sparse least-squares
  problem.
\newblock {\em SIAM Journal on Optimization}, 23(2):1062--1091.

\bibitem[Yang et~al., 2010]{yang2010fast}
Yang, A.~Y., Sastry, S.~S., Ganesh, A., and Ma, Y. (2010).
\newblock Fast $\ell_1$-minimization algorithms and an application in robust
  face recognition: A review.
\newblock In {\em 17th IEEE International Conference on Image Processing
  (ICIP)}, pages 1849--1852. IEEE.

\bibitem[Yi and Huang, 2017]{yi2017semismooth}
Yi, C. and Huang, J. (2017).
\newblock Semismooth newton coordinate descent algorithm for elastic-net
  penalized huber loss regression and quantile regression.
\newblock {\em Journal of Computational and Graphical Statistics},
  26(3):547--557.

\bibitem[Yun, 2014]{yun2014iteration}
Yun, S. (2014).
\newblock On the iteration complexity of cyclic coordinate gradient descent
  methods.
\newblock {\em SIAM Journal on Optimization}, 24(3):1567--1580.

\bibitem[Zhang, 2010]{zhang2010nearly}
Zhang, C.-H. (2010).
\newblock Nearly unbiased variable selection under minimax concave penalty.
\newblock {\em The Annals of Statistics}, 38(2):894--942.

\bibitem[Zhang and Huang, 2008]{zhang2008the}
Zhang, C.-H. and Huang, J. (2008).
\newblock The sparsity and bias of the lasso selection in high-dimensional
  linear regression.
\newblock {\em The Annals of Statistics}, 36(4):1567--1594.

\bibitem[Zhang, 2009]{zhang2009some}
Zhang, T. (2009).
\newblock Some sharp performance bounds for least squares regression with $l_1$
  regularization.
\newblock {\em The Annals of Statistics}, 37(5A):2109--2144.

\bibitem[Zhao and Yu, 2006]{zhao2006on}
Zhao, P. and Yu, B. (2006).
\newblock On model selection consistency of lasso.
\newblock {\em Journal of Machine Learning Research}, 7:2541--2563.

\bibitem[Zou and Hastie, 2005]{zou2005regularization}
Zou, H. and Hastie, T. (2005).
\newblock Regularization and variable selection via the elastic net.
\newblock {\em Journal of the Royal Statistical Society: Series B (Statistical
  Methodology)}, 67(2):301--320.

\end{thebibliography}
\addcontentsline{toc}{section}{References}
\end{document}